\documentclass[sigconf]{acmart}

\usepackage{balance}
\usepackage{adjustbox}
\usepackage{amssymb}
\usepackage{bm}
\usepackage{caption}
\usepackage{color}
\usepackage{makecell}
\usepackage{microtype}
\usepackage{multirow}
\usepackage{paralist}
\usepackage{rotating}
\usepackage{subcaption}
\usepackage{tabularx}
\usepackage{url}
\usepackage{xcolor}
\usepackage{enumitem}
\usepackage{booktabs}
\usepackage{algpseudocode}
\usepackage{listings}
\usepackage{hyperref}
\usepackage{xcolor}
\usepackage[normalem]{ulem}
\usepackage[linesnumbered,ruled]{algorithm2e}
\usepackage[capitalize]{cleveref}
\usepackage{tikz,nicefrac,amsmath,pifont,tkz-graph,makecell,tkz-graph}
\usetikzlibrary{arrows,backgrounds,patterns,matrix,shapes,fit,calc,shadows,plotmarks}
\usetikzlibrary{arrows.meta}

\usepackage[english]{babel}

\PassOptionsToPackage{dvipsnames}{xcolor}
\PassOptionsToPackage{hyphens}{url}
\definecolor{mylinkcolor}{RGB}{0,0,140}
\definecolor{ForestGreen}{RGB}{34,139,34}
\hypersetup{colorlinks,allcolors=mylinkcolor,citecolor=mylinkcolor}

\newcommand{\flowv}{\mathbf{f}}
\newcommand{\edgelap}{\mathbf{L}_e}

\newcommand{\xhdr}[1]{\vspace{0.5mm}\noindent{\textbf{#1.}}\hspace{0.5mm}}

\lstdefinelanguage{Julia}%
  {morekeywords={abstract,break,case,catch,const,continue,do,else,elseif,%
      end,export,false,for,function,immutable,import,importall,if,in,%
      macro,module,otherwise,quote,return,switch,true,try,type,typealias,%
      using,while},%
   sensitive=true,%
   alsoother={\$},%
   morecomment=[l]\#,%
   morecomment=[n]{\#=}{=\#},%
   morestring=[s]{"}{"},%
   morestring=[m]{'}{'},%
}[keywords,comments,strings]%

\def\BibTeX{{\rm B\kern-.05em{\sc i\kern-.025em b}\kern-.08emT\kern-.1667em\lower.7ex\hbox{E}\kern-.125emX}}
    
\copyrightyear{2019} 
\acmYear{2019} 
\setcopyright{acmlicensed}
\acmConference[KDD '19]{The 25th ACM SIGKDD Conference on Knowledge Discovery and Data Mining}{August 4--8, 2019}{Anchorage, AK, USA}
\acmBooktitle{The 25th ACM SIGKDD Conference on Knowledge Discovery and Data Mining (KDD '19), August 4--8, 2019, Anchorage, AK, USA}
\acmPrice{15.00}
\acmDOI{10.1145/3292500.3330872}
\acmISBN{978-1-4503-6201-6/19/08}

\begin{document}

\title{Graph-based Semi-Supervised \& Active Learning for Edge Flows}

\author{Junteng Jia}
\affiliation{\institution{Cornell University}}
\email{jj585@cornell.edu}

\author{Michael T. Schaub}
\affiliation{\institution{Massachusetts Institute of Technology}
\institution{University of Oxford}}
\email{mschaub@mit.edu}

\author{Santiago Segarra}
\affiliation{\institution{Rice University}}
\email{segarra@rice.edu}

\author{Austin R. Benson}
\affiliation{\institution{Cornell University}}
\email{arb@cs.cornell.edu}


\begin{abstract}
We present a graph-based semi-supervised learning (SSL) method for learning edge flows defined on a graph.
Specifically, given flow measurements on a subset of edges, we want to predict the flows on the remaining edges.
To this end, we develop a computational framework that imposes certain constraints on the overall flows, such as (approximate) flow conservation.
These constraints render our approach different from classical graph-based SSL for vertex labels, 
which posits that tightly connected nodes share similar labels and leverages the graph structure 
accordingly to extrapolate from a few vertex labels to the unlabeled vertices.

We derive bounds for our method's reconstruction error and demonstrate its
strong performance on synthetic and real-world flow networks from transportation, physical infrastructure, and the Web.
Furthermore, we provide two active learning algorithms for selecting informative edges on which to measure flow, which has applications for optimal sensor deployment.
The first strategy selects edges to minimize the reconstruction error bound
and works well on flows that are approximately divergence-free.
The second approach clusters the graph and selects bottleneck edges that cross cluster-boundaries,
which works well on flows with global trends.
\end{abstract}

%

\maketitle

\section{Introduction \label{sec:intro}}

\begin{figure}[t!]
\centering
\includegraphics[width=1.00\linewidth]{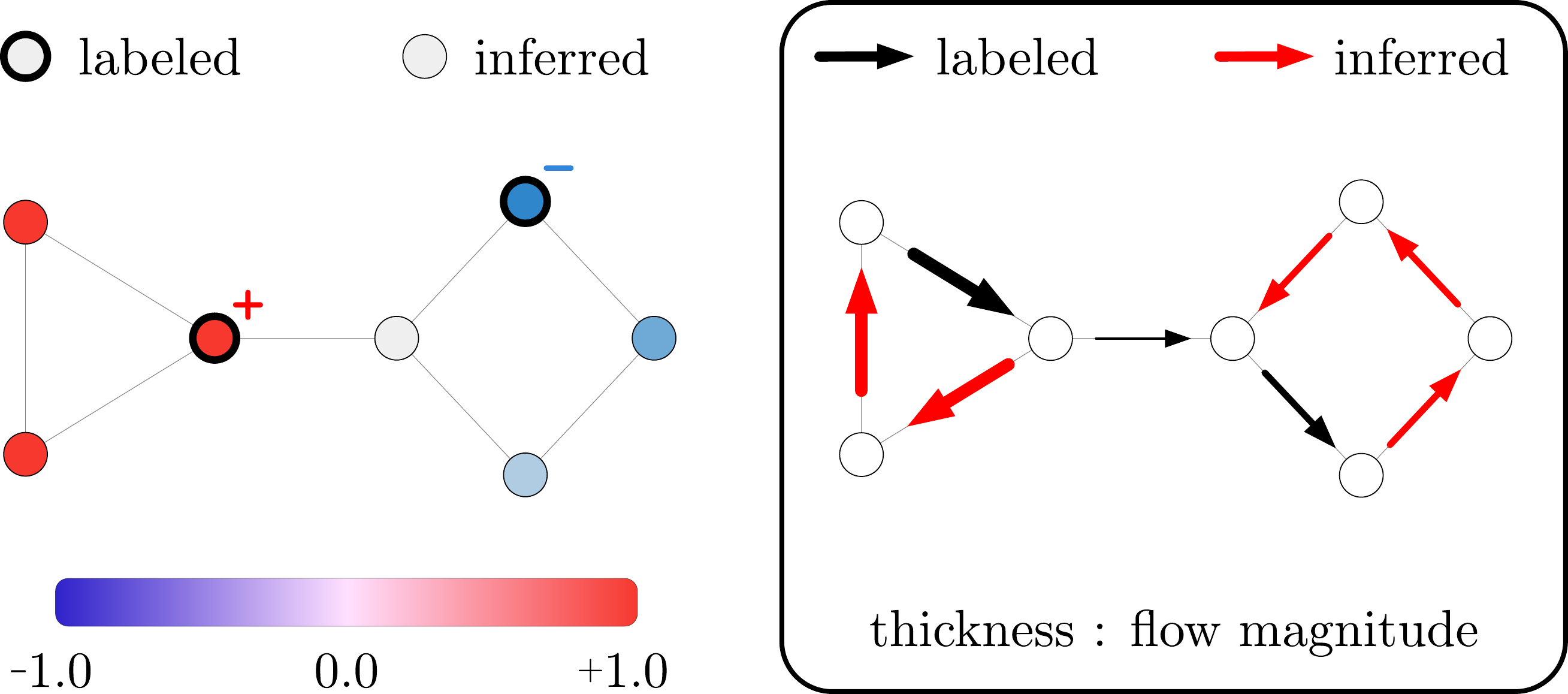}
\caption{Left: classical graph-based semi-supervised learning for vertex labels.
Right: Our framework of graph-based semi-supervised learning for edge flows.
}
\label{fig:intro}
\end{figure}
Semi-supervised learning (SSL) has been widely studied for large-scale data
mining applications, where the labeled data are often difficult, expensive, or
time consuming to obtain~\cite{Zhu_2009,Subramanya_2014}.
SSL utilizes both labeled and unlabeled data to improve prediction accuracy by enforcing a \textit{smoothness} constraint with respect to the intrinsic structure among all data samples.
Graph-based SSL is an important branch of semi-supervised learning.
It encodes the structure of data points with a similarity graph, where each vertex is a data sample and each edge is the similarity between a pair of vertices (\cref{fig:intro}, left).
Such similarity graphs can either be derived from actual relational data or be
constructed from data features using k-nearest-neighbors, $\epsilon$-neighborhoods
or Gaussian Random Fields~\cite{Zhu_2003,Joachims_2003}.
%
%
Graph-based SSL is especially suited for learning problems that are naturally defined on the \emph{vertices} of a graph, including social networks~\cite{Altenburger_2018}, web networks~\cite{Kyng_2015}, and co-purchasing networks~\cite{Gleich_2015}.
%
%

However, in many complex networks, the behavior of interest is a dynamical process on the \emph{edges}~\cite{schaub2014structure}, 
such as a flow of energy, signal, or mass.
%
%
For instance, in transportation networks, we typically monitor the traffic on roads (edges) that connect different intersections (vertices).
Other examples include energy flows in power grids, water flows in water supply networks, and data packets flowing between autonomous systems.
Similar to vertex-based data, edge flow data needs to be collected through dedicated sensors or special protocols and can
be expensive to obtain.
Although graph-theoretical tools like the line-graph~\cite{Godsil_2001} have been proposed to analyze graph-based data from an edge-perspective~\cite{Ahn_2010,evans2010line},  the problem of semi-supervised learning for edge flows has so far received little
attention, despite the large space of applications

%
Here we consider the problem of semi-supervised learning for edge flows for networks with fixed topology.
Given a network with a vertex set $\mathcal{V}$ and edge set $\mathcal{E}$, the (net) edge flows can be considered as real-valued alternating functions
$f\colon \mathcal{V} \times \mathcal{V} \rightarrow \mathbb{R}$, such that:
\begin{align}
f(i,j) =
\begin{cases}
-f(j,i), & \forall\ (i,j) \in \mathcal{E} \\
 0, & \text{otherwise}.
\end{cases}
\label{eq:flow_function}
\end{align}
As illustrated in \cref{fig:intro}, this problem is related to---yet fundamentally different from---SSL in the vertex-space.
Specifically, a key assumption underlying classical vertex-based SSL is that tightly-knit vertices are likely to share similar labels.
This is often referred to as the smoothness or cluster assumption~\cite{Chapelle_2003}.

However, naively translating this notion of smoothness to learn edge flows leads to sub-optimal algorithms.
As we will show in the following sections, applying classical vertex-based SSL to a line-graph~\cite{Godsil_2001}, which encodes the adjacency relationships between \emph{edges}, often produces worse results than simply ignoring the observed edges.
Intuitively, the reason is that smoothness is not the right condition for flow data: unlike a vertex label, each edge flow carries an orientation which is represented by the sign of its numerical flow value.
Enforcing smoothness on the line-graph requires the flow values on adjacent edges to be numerically close, which does \emph{not} reflect any insight into the underlying physical system.

%

To account for the different nature of edge data, we assume different kinds of SSL constraints for edge flows.
Specifically, we focus on flows that are almost conserved or \emph{divergence-free}---the total amount of flow that enters a vertex should approximately equal the flow that leaves.
Given an arbitrary network and a set of labeled edge flows $\mathcal{E}^{\rm L}$, the unlabeled edge flows on $\mathcal{E}^{\rm U}$ are inferred by minimizing a cost function based on the edge Laplacian $\edgelap$ that measures divergence at all vertices.
We provide the perfect recovery condition for strictly divergence-free flows, and derive an upper bound for the reconstruction error when a small perturbation is added.
We further show that this minimization problem can be converted into a linear least-squares problem and thus solved efficiently.
Our method substantially outperforms two competing baselines (including the line-graph approach) as measured by the Pearson correlation coefficient between the inferred edge flows and the observed ground truth on a variety of real-world datasets.

We further consider \emph{active} semi-supervised learning for edge flows, where we aim to select a fraction of edges that is most informative for inferring the flows on the remaining edges.
An important application of this problem is optimal sensor placement, where we want to deploy flow sensors on a limited number of edges such that the reconstructed edge flows are as accurate as possible.
We propose two active learning strategies and demonstrate substantial performance gains over random edge selection.
Finally, we discuss how our methods can be extended to other types of structured edge flows by highlighting connections with algebraic topology.
We summarize our main contributions as follows: 
(1) a semi-supervised learning method for edge flows;
(2) two active learning algorithms for choosing informative edges; and
(3) analysis of real-world data that demonstrate the superiority of our method to alternatives.

\section{Methodology \label{sec:method}}
Given an undirected network with vertex set $\mathcal{V}$, edge set $\mathcal{E}$, and a labeled set of edge flows on
$\mathcal{E}^{\rm L} \subseteq \mathcal{E}$, our goal is to predict the unlabeled edge flows
$\mathcal{E}^{\rm U} \equiv \mathcal{E} \backslash \mathcal{E}^{\rm L}$.
Although our assumption about the data in this problem is different from classical graph-based SSL for vertex labels, 
the associated matrix computations in the two problems have striking similarities. 
In fact, we show in~\cref{sec:topology} that the similarity is mediated by deep connections to algebraic topology.
In this section, we first review classical vertex-based SSL and then show how it relates to our edge-based method.
\Cref{tab:symbols} summarizes notation used throughout the paper.

\xhdr{Background on Graph-based SSL for Vertex Labels}
In the SSL problem for vertex labels, we are given the labels of a subset of vertices $\mathcal{V}^{\rm L}$, and our goal is to find a label assignment of the unlabeled vertices $\mathcal{V}^{\rm U}$ such that the labels vary smoothly across neighboring vertices.
Formally, this notion of smoothness (or the deviation from it, respectively) can be defined via a loss function of the form\footnote{All norms for vectors and matrices in this paper are the 2-norm.}
\begin{align} \textstyle
\|\mathbf{B}^{\intercal} \mathbf{y}\|^{2} = \sum_{(i,j) \in \mathcal{E}} (y_{i} - y_{j})^{2},
\label{eq:loss_smoothness}
\end{align}
where $\mathbf{y}$ is the vector containing vertex labels, and $\mathbf{B} \in \mathbb{R}^{n \times m}$ is the incidence matrix of the network, defined as follows.
Consider the edge set $\mathcal{E} = \{\mathcal{E}_{1}, \ldots, \mathcal{E}_{r}, \ldots, \mathcal{E}_{m}\}$ and, without loss of generality, choose a reference orientation for every edge such that it points from the vertex with the smaller index to the vertex with the larger index.
Then the incidence matrix $\mathbf{B}$ is defined as 
\begin{align}
B_{kr} =
\begin{cases}
 1, & \mathcal{E}_{r} \equiv (i,j),\ k = i,\ i < j \\
-1, & \mathcal{E}_{r} \equiv (i,j),\ k = j,\ i < j \\
 0, & \text{otherwise}.
\end{cases}
\label{eq:incidence_matrix}
\end{align}
The loss function in \cref{eq:loss_smoothness} is the the sum-of-squares label difference 
between all connected vertices.
The loss can be written compactly as $\|\mathbf{B}^\intercal\mathbf{y}\|^2 = \mathbf{y}^\intercal\mathbf{L}\mathbf{y}$ in terms of the graph Laplacian $\mathbf{L} = \mathbf{B} \mathbf{B}^{\intercal}$.

In vertex-based SSL, unknown vertex labels are inferred by minimizing the quadratic form $\mathbf{y}^{\intercal} \mathbf{L} \mathbf{y}$
with respect to $\mathbf{y}$ while keeping the labeled vertices fixed.%
\footnote{This is the formulation of Zhu, Ghahramani, and Lafferty~\cite{Zhu_2003}. There are other
  graph-based SSL methods~\cite{Subramanya_2014}; however, most of them employ a similar loss-function based on variants of the graph Laplacian $\mathbf{L}$.}
Using $\hat{y}_i$ to denote an observed label on vertex $i$, the optimization problem is:
\begin{align}
\mathbf{y}^{*} = \arg \min_{\mathbf{y}} \|\mathbf{B}^{\intercal} \mathbf{y}\|^{2} \qquad \text{s.t.} \quad \ y_{i} = \hat{y}_{i}, \, \forall \mathcal{V}_{i} \in \mathcal{V}^{\rm L}.
\label{eq:objective_smoothness}
\end{align}
For connected graphs with more edges than vertices ($m > n$),
\cref{eq:objective_smoothness} has a unique solution provided at least one vertex is labeled.
%

\subsection{Graph-Based SSL for Edge Flows}
We now consider the SSL problem for edge flows.
The edge flows over a network can be represented with a vector $\flowv$, where
$\mathrm{f}_{r} > 0$ if the flow orientation on edge $r$ aligns with its
reference orientation and $\mathrm{f}_{r} < 0$ otherwise.
In this sense, we are only accounting for the \emph{net flow} along an edge.
We denote the ground truth (measured) edge flows in the network as $\hat{\flowv}$.
To impose a flow conservation assumption for edge flows, 
we consider the divergence at each vertex, 
which is the sum of outgoing flows minus the sum of incoming flows at a vertex.
For arbitrary edge flows $\flowv$, the divergence on a vertex $i$ is
\[
(\mathbf{B} \mathbf{f})_{i} = \sum_{\mathcal{E}_{r} \in \mathcal{E} \;:\; \mathcal{E}_{r} \equiv (i,j), i < j} \mathrm{f}_{r} -
\sum_{\mathcal{E}_{r} \in \mathcal{E} \;:\; \mathcal{E}_{r} \equiv (j,i), j < i} \mathrm{f}_{r}.
\]
To create a loss function for edge flows that enforces a notion of flow-conservation,
we use the sum-of-squares vertex divergence:
\begin{align}
\|\mathbf{B} \flowv\|^{2} = \flowv^{\intercal} \mathbf{B}^{\intercal}\mathbf{B} \flowv = \flowv^{\intercal} \edgelap \flowv.
\label{eq:loss_divergence_free}
\end{align}

Here $\edgelap = \mathbf{B}^{\intercal} \mathbf{B}$ is the so-called edge Laplacian matrix.
Interestingly, the loss function for penalizing divergence contains the transpose of the incidence matrix $\mathbf{B}$, which appeared in the measure of smoothness in the vertex-based problem [cf.\ \cref{eq:loss_smoothness}]. 
However, unlike the case for smooth vertex labels, requiring $\flowv^{\intercal}\edgelap \flowv = 0$ is actually under-constrained, i.e., even when more than one edge is labeled, many different divergence-free edge-flow assignments may exist that induce zero loss.
We thus propose to regularize the problem and solve the following constrained optimization problem:
\begin{align}
\flowv^{*} = \arg \min_{\flowv} \|\mathbf{B} \flowv\|^{2} + \lambda^{2} \cdot \|\flowv\|^{2} 
               \qquad \text{s.t.} \quad \mathrm{f}_{r} = \hat{\mathrm{f}}_{r}, \, \forall \mathcal{E}_{r} \in \mathcal{E}^{\rm L}.
\label{eq:objective_divergence_free}
\end{align}
The first term in the objective function is the loss, while the second term is a regularizer that guarantees a unique optimal solution.

\begin{table}[t]
\caption{Summary of notation used throughout the paper.}
\begin{tabular}{r l}
\toprule
\textbf{Symbol} & \textbf{Description} \\
\hline
$n \in \mathbb{N}$  & $\lvert \mathcal{V}\rvert$ = number of vertices \\
$m \in \mathbb{N}$  & $\lvert \mathcal{E} \rvert$ = number of edges \\
$o \in \mathbb{N}$  & $\lvert \mathcal{T} \vert$ = number of triangles \\
$m^{\rm L} \in \mathbb{N}$ & $\lvert \mathcal{E}^{\rm L} \rvert$ = number of labeled edges \\
$m^{\rm U} \in \mathbb{N}$ & $\lvert \mathcal{E}^{\rm U} \rvert$ = number of unlabeled edges \\
$c \in \mathbb{N}$  & $m-n+1$ = number of independent cycles \\
$i,j,k \in \mathbb{N}$ & vertex index \\
$r,s,t \in \mathbb{N}$ & edge index \\
$u \in \mathbb{N}$ & triangle index \\
$\alpha,\beta,\gamma \in \mathbb{N}$ & index for spectral coefficient \\
\midrule
$\mathbf{y}, \hat{\mathbf{y}} \in \mathbb{R}^{n}$ & vertex labels, ground truth vertex labels \\
$\flowv, \hat{\flowv} \in \mathbb{R}^{m}$ & edge flows, ground truth edge flows \\
$\mathbf{w} \in \mathbb{R}^{o}$ & function defined on triangles \\
\midrule
$\mathbf{B} \in \mathbb{R}^{n \times m}$ & node-edge incidence matrix (see \cref{eq:incidence_matrix}) \\
$\mathbf{C} \in \mathbb{R}^{m \times o}$ & edge-triangle curl matrix  (see \cref{eq:curl_matrix}) \\
$\mathbf{L} \in \mathbb{R}^{n \times n}$ & Laplacian $\mathbf{L} = \mathbf{B} \mathbf{B}^{\intercal}$ \\
$\edgelap \in \mathbb{R}^{m \times m}$ & edge Laplacian $\edgelap = \mathbf{B}^{\intercal} \mathbf{B}$ \\
\bottomrule
\end{tabular}
\label{tab:symbols}
\end{table}

\xhdr{Computation} The equality constraints in \cref{eq:objective_divergence_free} can be
eliminated by reducing the number of free variables.
Let $\flowv^{0}$ be a trivial feasible point for \cref{eq:objective_divergence_free}
where $\mathrm{f}^{0}_{r} = \hat{\mathrm{f}}_{r}$ if $r \in \mathcal{E}^{\rm L}$
and $\mathrm{f}^{0}_{r} = 0$ otherwise.
Moreover, denote the set of indices for unlabeled edges as $\mathcal{E}^{\rm U} = \{\mathcal{E}_{1}^{\rm U}, \mathcal{E}_{2}^{\rm U}, \ldots, \mathcal{E}_{m^{\rm U}}^{\rm U}\}$. We define the expansion operator $\mathbf{\Phi}$ as a linear map from $\mathbb{R}^{m^{\rm U}}$ to $\mathbb{R}^{m}$ given by
$ \Phi_{rs} = 1$ if $\mathcal{E}_{r} = \mathcal{E}_{s}^{\rm U}$ and $0$ otherwise.
Let $\flowv^{\rm U} \in \mathbb{R}^{m^{\rm U}}$ be the edge flows on the unlabeled edges.
Any feasible point for \cref{eq:objective_divergence_free} can be written as $\flowv^{0} + \mathbf{\Phi} \flowv^{\rm U}$,
and the original problem can be converted to a linear least-squares problem:
\begin{align}
\flowv^{\rm U*} = \arg \min_{\flowv^{\rm U}}
\left\|
\begin{bmatrix}
\mathbf{B} \mathbf{\Phi} \\
\lambda \cdot \mathbf{I} \\
\end{bmatrix}
\flowv^{\rm U}
-
\begin{bmatrix}
-\mathbf{B} \flowv^{0}\\
0
\end{bmatrix}
\right\|^{2}.
\label{eq:least_square}
\end{align}
Typically, $\mathbf{B}$ is a large sparse matrix.
Thus, the least-squares problem in \cref{eq:least_square} can be solved with iterative methods such as LSQR~\cite{Paige_1982} or LSMR~\cite{Fong_2011}, which is guaranteed to converge in $m^{\rm U}$ iterations.
Those iterative solvers use sparse matrix-vector multiplication as subroutine, with $\mathcal{O}(m)$ computational cost per iteration.
By choosing $\lambda > 0$, \cref{eq:least_square} can be made
well-conditioned, and the iterative methods will only take a small number
of iterations to converge.
\subsection{Spectral Graph Theory Interpretations \label{subsec:theory}}
\begin{figure}[t]
\centering
\includegraphics[width=1.00\linewidth]{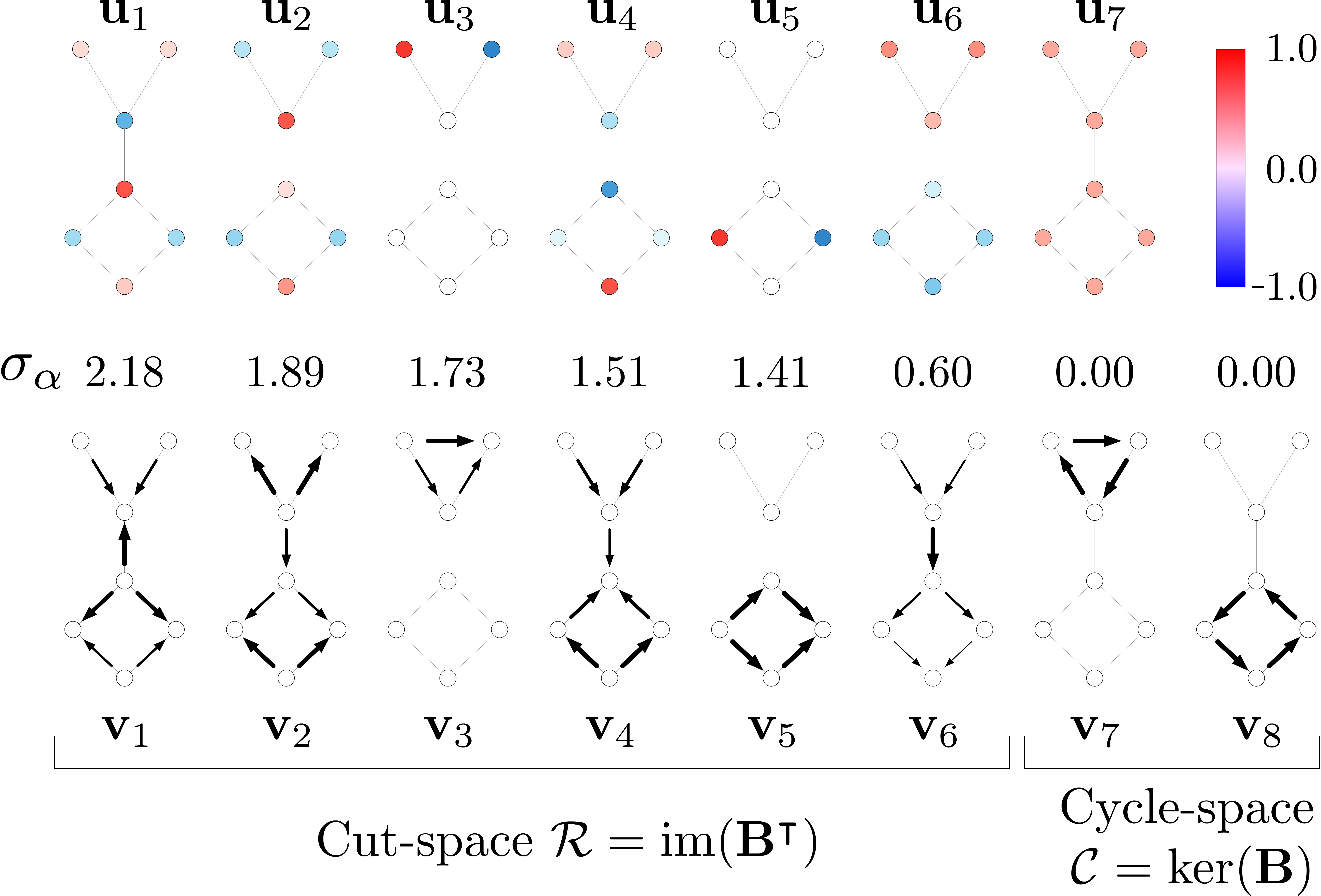}
\vspace{-4.75mm}
\caption{Singular vectors for the incidence matrix $\mathbf{B}$ of an example graph.
Top: the left singular vectors form a basis for vertex labels. Numerical values are encoded by color of the vertices.
Middle: singular values represent the ``frequencies'' of left singular vectors or the divergences of right singular vectors.
Bottom: right singular vectors form a basis for edge flows, where the arrow points to the flow direction and the edge-width encodes the magnitude of the flow.
}
\label{fig:spectral}
\end{figure}
We first briefly review graph signal processing in the
vertex-space before introducing similar tools to deal with edge flows.
The eigenvectors of the Laplacian matrix have been widely used in graph signal processing for vertex labels, since the corresponding eigenvalues carry a notion of frequency that provides a sound mathematical and intuitive basis for
analyzing functions on vertices~\cite{Shuman_2013,Ortega_2018}.
The spectral decomposition of the graph Laplacian matrix is
$\mathbf{L} = \mathbf{U}\ \mathbf{\Lambda}\ \mathbf{U}^{\intercal}$.
Because $\mathbf{L} = \mathbf{B} \mathbf{B}^{\intercal}$, the orthonormal basis $\mathbf{U} \in \mathbb{R}^{n \times n}$ for vertex labels is formed by the \emph{left} singular vectors of the incidence matrix:
$\mathbf{B} = \mathbf{U}\ \mathbf{\Sigma}\ \mathbf{V}^{\intercal}$,
where $\mathbf{\Sigma} \in \mathbb{R}^{n \times m}$ is the diagonal matrix of ordered singular values with $m-n$ columns of zero-padding on the right, and the right singular vectors $\mathbf{V} \in \mathbb{R}^{m \times m}$ is an orthonormal basis for edge flows.
To simplify our discussion, we will say that the basis vectors in the last $m-n$ columns of $\mathbf{V}$ also have singular value $0$.

The divergence-minimizing objective in \cref{eq:objective_divergence_free} can be
rewritten in terms of the right singular vectors of $\mathbf{B}$, thus providing a formal connection between the vertex-based and the edge-based SSL problem.
Let $\mathbf{p} = \mathbf{V}^\intercal\flowv \in \mathbb{R}^{m}$ represent the \emph{spectral coefficients} of
$\flowv$ expressed in terms of the basis $\mathbf{V}$.
Then, we can rewrite
\cref{eq:objective_divergence_free} as
\begin{align}
\flowv^{*} &= \mathbf{V} \cdot \arg \min_{\mathbf{p}}\ (\mathbf{V} \mathbf{p})^{\intercal} \mathbf{B}^{\intercal} \mathbf{B}\ (\mathbf{V} \mathbf{p}) + \lambda^{2} \cdot (\mathbf{V} \mathbf{p})^{\intercal} (\mathbf{V} \mathbf{p}) \nonumber \\ 
&= \mathbf{V} \cdot \arg \min_{\mathbf{p}} \mathbf{p}^{\intercal} \left(\mathbf{\Sigma}^{\intercal} \mathbf{\Sigma} + \lambda^{2} \cdot \mathbf{I}\right) \mathbf{p} \nonumber \\ 
&= \mathbf{V} \cdot \arg \min_{\mathbf{p}} \lambda^{2} \cdot \sum_{\alpha} \frac{\sigma_{\alpha}^{2}+\lambda^{2}}{\lambda^{2}} p_{\alpha}^{2} \nonumber \\ 
&\hspace{1.35cm} \text{s.t.} \quad (\mathbf{V} \mathbf{p})_{r}=\hat{\mathrm{f}}_{r}, \, \forall \mathcal{E}_{r} \in \mathcal{E}^{\rm L},
\label{eq:objective_spectral}
\end{align}
which minimizes the weighted sum-of-square of the spectral coefficients under equality constraints for measured edge flows.

\xhdr{Signal smoothness, cut-space, and cycle space}
The connection between the vertex and edge-based problem is in fact not just a formal relationship, but can be given a clear (physical) interpretation.
By construction $\mathbf{V}$ is a complete orthonormal basis for the space of edge flows (here identified with $\mathbb{R}^m$).
This space can be decomposed into two orthogonal subspaces (see also~\cref{fig:spectral}).

The first subspace is the \emph{cut-space} $\mathcal R = \text{im}(\mathbf{B}^\intercal)$~\cite{Godsil_2001}, spanned by the singular vectors $\mathbf{V}_{\mathcal{R}}$ associated with nonzero singular values.
The space $\mathcal{R}$ is also called the space of gradient flows, since any vector may be written as $\mathbf{B}^\intercal \mathbf{y}$,
where $\mathbf{y}$ is a vector of vertex scalar potentials that induce a gradient flow.
The second subspace is the \emph{cycle-space} $\mathcal C = \text{ker}(\mathbf{B})$~\cite{Godsil_2001}, spanned by the remaining right singular vectors $\mathbf{V}_{\mathcal{C}}$ associated with zero singular values.
Note that any vector $\mathbf{f}\in \mathcal{C}$ corresponds to a circulation of flow, and will induce zero cost in our loss function \cref{eq:loss_divergence_free}.
In fact, for a connected graph, the number of right singular vectors with zero singular values equals $c = m - n + 1$, which is the number of independent cycles in the graph.
We denote the spectral coefficients for basis vectors in these two spaces as $\mathbf{p}_{\mathcal{R}} \in \mathbb{R}^{m-c}$ and $\mathbf{p}_{\mathcal{C}} \in \mathbb{R}^{c}$.

Let $\mathbf{u}_{\alpha}, \sigma_{\alpha}, \mathbf{v}_{\alpha}$ denote a triple of a left
singular vector, singular value, and right singular vector.
The singular values
$\mathbf{u}_{\alpha}^{\intercal}\ \mathbf{L}\ \mathbf{u}_{\alpha} = \sigma_{\alpha}^{2}$ 
provide a notion of ``unsmoothness'' of basis vector
$\mathbf{u}_{\alpha}$ representing vertex labels, while
$\mathbf{v}_{\alpha}^{\intercal}\ \edgelap\ \mathbf{v}_{\alpha} =
\sigma_{\alpha}^{2}$ gives the sum-of-squares divergence of basis vector
$\mathbf{v}_{\alpha}$ representing edge flows.
As an example, \cref{fig:spectral} displays the left and right singular vectors of a small graph.
The two singular basis vectors $\mathbf{v}_{7}$ and $\mathbf{v}_{8}$ associated with zero singular values correspond to cyclic edge flows in the network.
The remaining flows $\mathbf{v}_{i}$, $i=1,\ldots, 6$---corresponding to non-zero singular values---all have a non-zero divergence.
Note also how the left singular vectors $\mathbf{u}_{i}$ associated with non-zero singular values give rise to the right singular vectors $\mathbf{v}_{i} = 1/\sigma_{i} \cdot \mathbf{B}^\intercal\mathbf{u}_{i}$ for $i = 1,\ldots 6$. As the singular vectors $\mathbf{u}_{i}$ can be interpreted as potential on the nodes, this highlights that the cut-space is indeed equivalent to the space of gradient flows (note that $\mathbf{u}_{7}$ induces no gradient).

\subsection{Exact and perturbed recovery}
From the above discussion, we can derive an exact recovery condition for the edge flows in the divergence-free setting.
\begin{lemma}
  Assume the ground truth flows are divergence-free.
  Then as $\lambda \rightarrow 0$, the solution of \cref{eq:objective_spectral} 
  can exactly recover the ground truth from some labeled edge set $\mathcal{E}^{\rm L}$ with cardinality $c=m-n+1$.
\label{thm:exact_recovery}
\end{lemma}
\begin{proof}
If the ground truth edge flows are divergence free (cyclic), 
the spectral coefficients of the basis vectors with non-zero singular values must be zero.
Recall that $\mathbf{p}_{\mathcal{C}} \in \mathbb{R}^{c}$ are the spectral coefficients of a basis $\mathbf{V}_{\mathcal{C}}$
in the cycle-space, then the ground truth edge flows can be written as $\hat{\flowv} = \mathbf{V}_{\mathcal{C}} \mathbf{p}_{\mathcal{C}}$
(the singular vectors $\mathbf{V}_\mathcal{C}$ that form the basis of the the cycle space are not unique as the singular values are 
``degenerate''; any orthogonal transformation is also valid).
On the other hand, in the limit $\lambda \rightarrow 0$, the spectral
coefficients of basis vectors with non-zero singular values have infinite
weights and are forced to zero [cf. \cref{eq:objective_spectral}].
Therefore, by choosing the set of labeled edges corresponding to $c=m-n+1$
linearly independent rows from $\mathbf{V}_{\mathcal{C}}$, the ground truth $\hat{\flowv}$ is the unique optimal solution.
\end{proof}
Furthermore, when a perturbation is added to divergence-free edge flows $\flowv$, the reconstruction error can be bounded as follows.
\begin{theorem}
  Let $\mathbf{V}_{\mathcal{C}}^{\rm L}$ denote $c$ linearly independent rows of the $\mathbf{V}_{\mathcal{C}}$ that correspond to labeled edges. 
  If the divergence-free edge flows $\flowv$ are perturbed by $\delta$, then as $\lambda \rightarrow 0$, the reconstruction error of the proposed algorithm is bounded by
  $[\sigma_{\rm min}^{-1}(\mathbf{V}_{\mathcal{C}}^{\rm L}) + 1] \cdot \|\mathbf{\delta}\|$.
\label{lma:reconstruction_error}
\end{theorem}
\begin{proof}
The ground truth edge flows can be written as,
\begin{align}
\hat{\flowv} = \flowv + \mathbf{\delta} =
\begin{bmatrix}
\flowv^{\rm L} \\
\flowv^{\rm U} \\
\end{bmatrix}
+
\begin{bmatrix}
\mathbf{\delta}^{\rm L} \\
\mathbf{\delta}^{\rm U} \\
\end{bmatrix},
\end{align}
where $\flowv^{\rm L}$, $\flowv^{\rm U}$ are the divergence-free edge flows on labeled and unlabeled edges, while $\mathbf{\delta}^{\rm L}$, $\mathbf{\delta}^{\rm U}$ are the corresponding perturbations.
Further, the reconstructed edge flows from \cref{eq:objective_spectral} are given by $\mathbf{V}_{\mathcal{C}} (\mathbf{V}_{\mathcal{C}}^{\rm L})^{-1}(\flowv^{\rm L} + \delta^{\rm L})$.
Therefore, we can bound the norm of the reconstruction error as follows:
\begin{align}
&&\|\mathbf{V}_{\mathcal{C}} (\mathbf{V}_{\mathcal{C}}^{\rm L})^{-1} (\flowv^{\rm L} + \mathbf{\delta}^{\rm L}) - (\flowv + \mathbf{\delta})\| 
&= \|\mathbf{V}_{\mathcal{C}} (\mathbf{V}_{\mathcal{C}}^{\rm L})^{-1} \mathbf{\delta}^{\rm L} - \mathbf{\delta}\| \nonumber \\
&&\leq \|\mathbf{V}_{\mathcal{C}} (\mathbf{V}_{\mathcal{C}}^{\rm L})^{-1} \mathbf{\delta}^{\rm L}\| + \|\mathbf{\delta}\| 
&= \|(\mathbf{V}_{\mathcal{C}}^{\rm L})^{-1} \mathbf{\delta}^{\rm L}\| + \|\mathbf{\delta}\| \nonumber \\
&&\leq \|(\mathbf{V}_{\mathcal{C}}^{\rm L})^{-1}\| \cdot \|\mathbf{\delta}^{\rm L}\| + \|\mathbf{\delta}\| 
&\leq [\|(\mathbf{V}_{\mathcal{C}}^{\rm L})^{-1}\| + 1] \cdot \|\mathbf{\delta}\|.
\label{eq:error_bound}
\end{align}

The first equality in \cref{eq:error_bound} comes from \cref{thm:exact_recovery}, and the second equality is due to the orthonormal columns of $\mathbf{V}_{\mathcal{C}}$.
Finally, the norm of a matrix equals its largest singular value,
and the singular values of the matrix inverse are the reciprocals of the singular
values of the original matrix.
Therefore, we can rewrite \cref{eq:error_bound} as follows
\begin{align}
[\|(\mathbf{V}_{\mathcal{C}}^{\rm L})^{-1}\| + 1] \cdot \|\mathbf{\delta}\| &= [\sigma_{\rm max}((\mathbf{V}_{\mathcal{C}}^{\rm L})^{-1}) + 1] \cdot \|\mathbf{\delta}\| \nonumber \\
&= [\sigma_{\rm min}^{-1}(\mathbf{V}_{\mathcal{C}}^{\rm L}) + 1] \cdot \|\mathbf{\delta}\|.
\end{align}
\end{proof}


\section{Semi-Supervised Learning Results \label{sec:ssl_result}}
Having discussed the theory and computations underpinning our method, we now examine
its application on a collection of networks with synthetic and real-world edge flows.
As our method is based on a notion of divergence-free edge flows, naturally we find
the most accurate edge flow estimates when this assumption approximately
holds.
For experiments in this section, the labeled sets of edges are chosen uniformly at random.
In \cref{sec:acl}, we provide active learning algorithms for selecting \emph{where} to measure.

\subsection{Learning Synthetic Edge Flows \label{subsec:synthetic}}
\begin{figure}[t]
\begin{subfigure}{0.53\linewidth}
\includegraphics[width=1.0\linewidth]{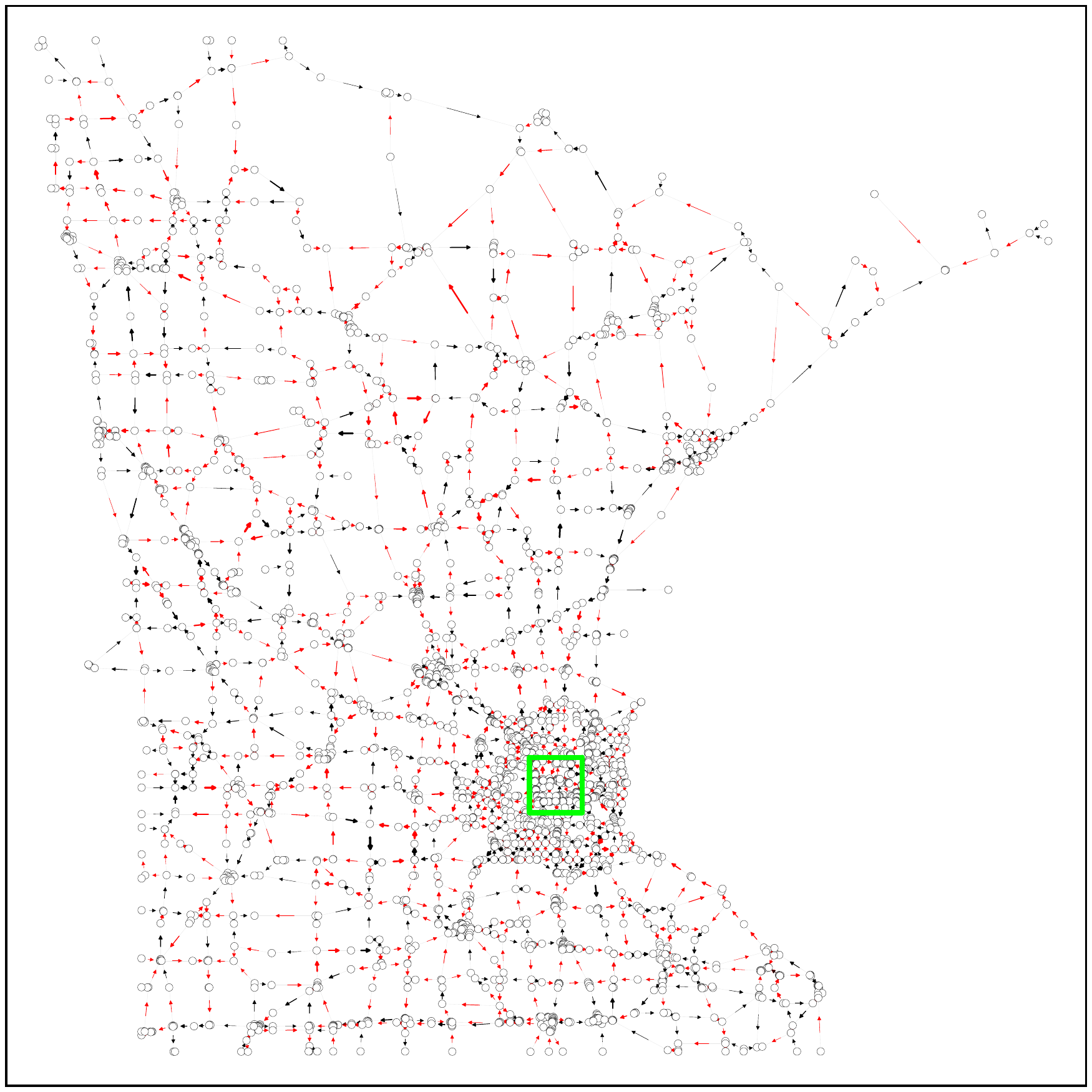}
\end{subfigure}
\hfill
\begin{subfigure}{0.46\linewidth}
\includegraphics[width=1.0\linewidth]{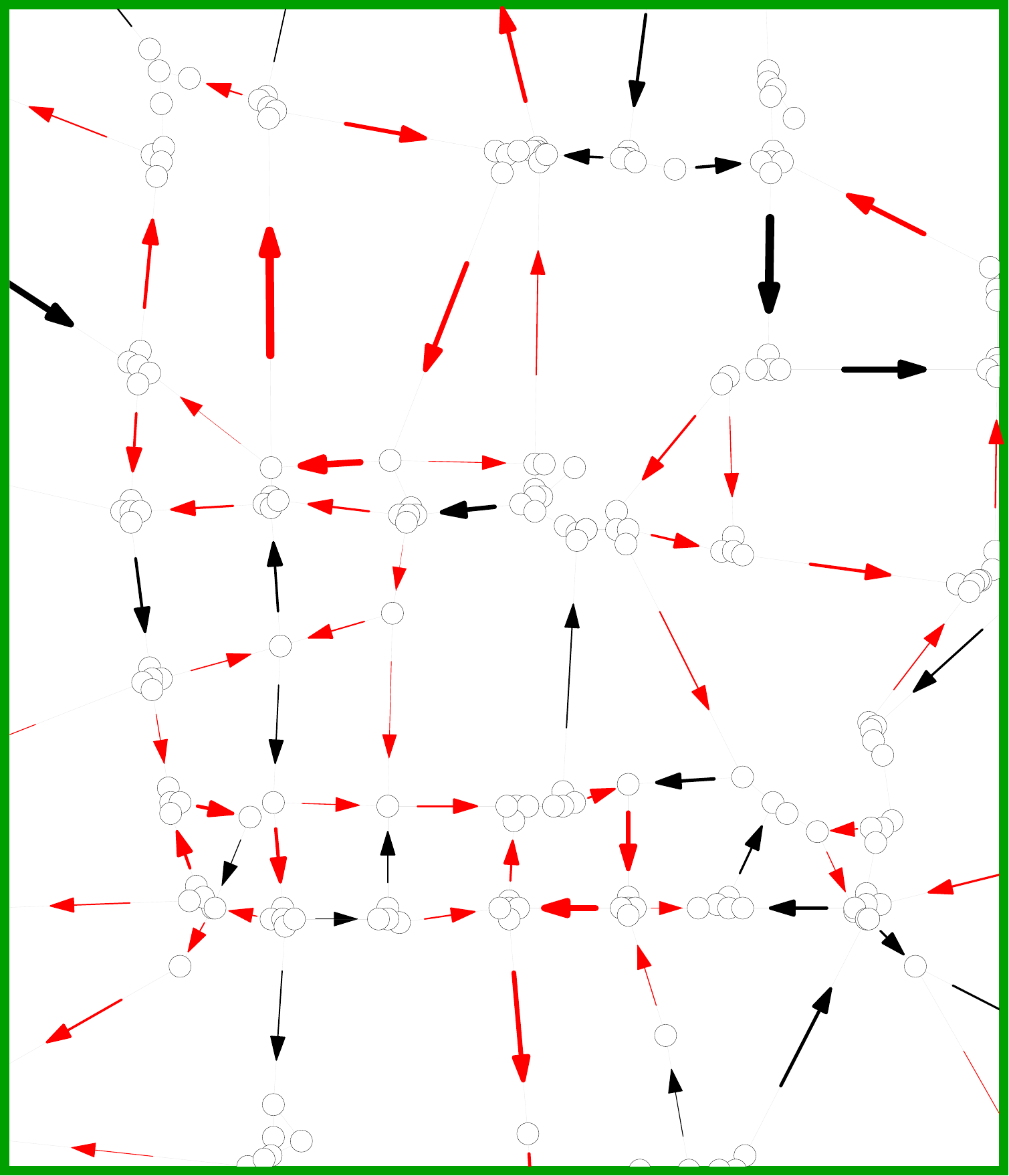}
\end{subfigure}
\caption{Synthetic traffic flow in Minnesota road network; 
$40\%$ of the edges are labeled and their flow is plotted in black.
The remaining red edge flows are inferred with our algorithm.
The width of each arrow is proportional to the magnitude of flow on the edge. The Pearson correlation coefficient between the inferred flows $\flowv^{*}$ and the ground truth $\hat{\flowv}$ is 0.956.
}
\label{fig:minnesota}
\end{figure}
\begin{figure}[t]
\centering
\includegraphics[width=1.0\linewidth]{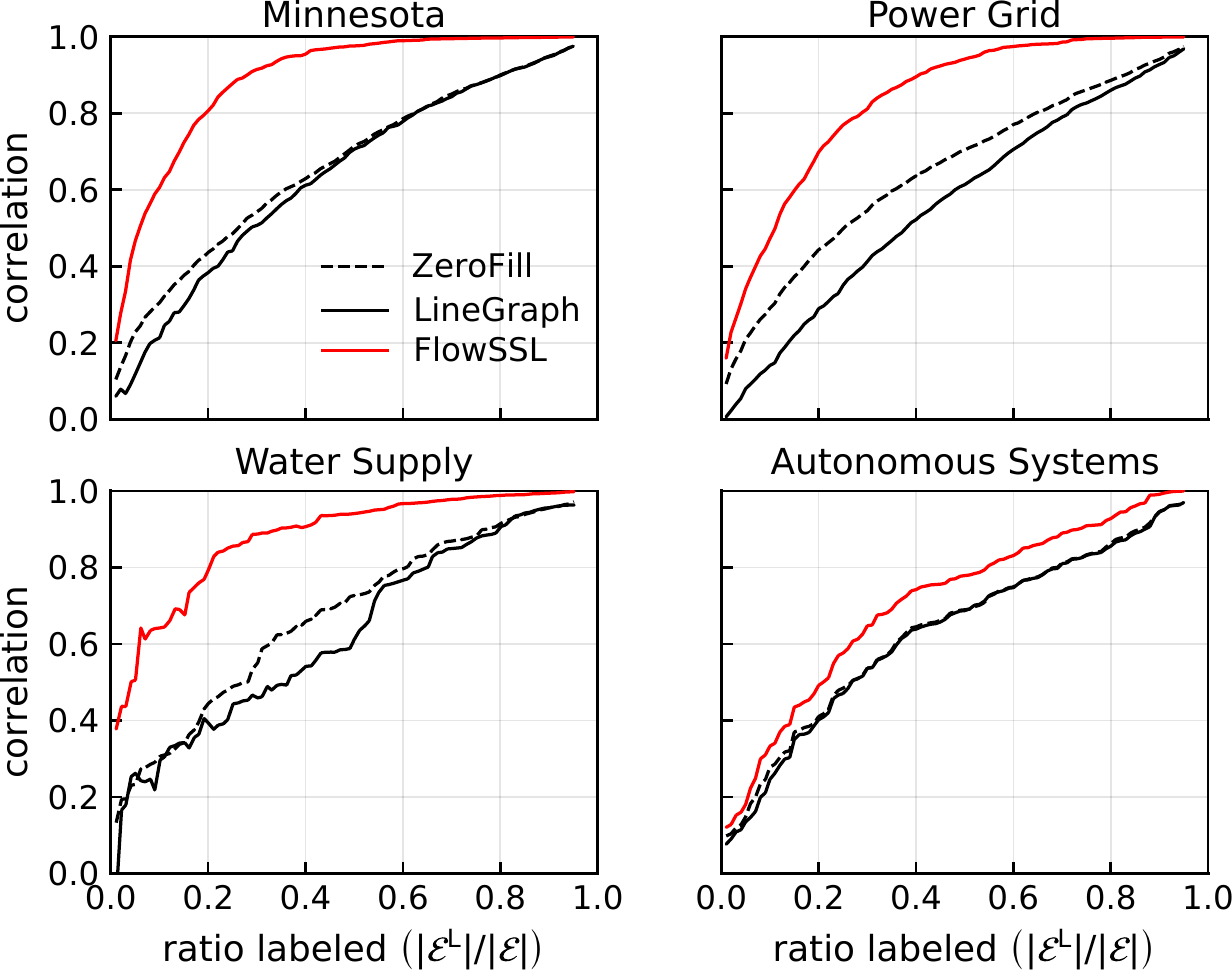}
\caption{Graph-based SSL for synthetic flows.
The plots show the correlation between the estimated flow vector $\flowv^{*}$ and 
the synthetic ground truth edge flows $\hat{\flowv}$ as a function of the ratio of labeled edges.
}
\label{fig:synthetic}
\end{figure}
\xhdr{Flow Network Setup}
In our first set of experiments, the network topology comes from real data, but we use
synthetic flows to demonstrate our method.
Later, we examine edge flows from real-world measurements.
We use the following four network topologies in our synthetic flow examples:
(1) The Minnesota road network where edges are roads and vertices are intersections ($n = 2642, m = 3303$)~\cite{Gleich-2009-thesis};
(2) The US power grid network from KONECT, where vertices are power stations or individual consumers and edges are transmission lines ($n = 4941, m = 6593$)~\cite{Kunegis_2013};
(3) The water irrigation network of Balerma city in Spain where vertices are water supplies or hydrants and edges are water pipes ($n = 447, m = 454$)~\cite{Reca_2006}; and
(4) An autonomous system network ($n = 520, m = 1280$)~\cite{Leskovec_2005}.

For each network, we first perform an SVD on its incidence matrix to get
the edge-space basis vectors $\mathbf{V}$.
The synthetic edge flows are then created by specifying the spectral coefficients $\mathbf{p}$, i.e., the mixture of these basis vectors.
Recall from \cref{subsec:theory} that the singular values associated with the basis vectors measure the magnitude of the divergence of each of these flow vectors.
To obtain a divergence-free flow, the spectral coefficients for all basis vectors $\mathbf{V}_\mathcal{R}$ spanning the cut space (associated with a nonzero singular value) should thus be set to zero.
However, to mimic the fact that  most real-world edge flows are not perfectly divergence-free, we do not set the spectral coefficients for the basis vectors in $\mathbf{V}_\mathcal{R}$ to zero.
Instead, we create synthetic flows with spectral coefficients for each basis vector (indexed by
$\alpha$) that are inversely proportional to the associated singular values $\sigma_{\alpha}$:
\begin{align}\label{eq:synth_spectral}
p_{\alpha} = \frac{b}{\sigma_{\alpha} + \epsilon},
\end{align}
where $b$ is a parameter that controls the overall magnitude of the edge flows and $\epsilon$ is a damping factor.
We choose $b = 0.02$, $\epsilon = 0.1$ in all examples shown in this paper.

\xhdr{Performance Measurement and Baselines}
Using the synthetic edge flows as our ground truth $\hat{\flowv}$, we conduct numerical experiments by selecting a fraction of edges \emph{uniformly at random} as the labeled edges $\mathcal{E}^{\rm L}$, and using our method to infer the edge flow on the unlabeled edges $\mathcal{E}^{\rm U}$.
To quantify the accuracy of the inferred edge flows, we use the Pearson correlation coefficient $\rho$ between the ground truth edge flows $\hat{\flowv}$ and the inferred edge flow $\flowv^{*}$.\footnote{Consistent results are obtained with other accuracy metrics, e.g., the relative $L^{2}$ error.}
The regulation parameter $\lambda$ in \cref{eq:objective_divergence_free} is $0.1$.
To illustrate the results, \cref{fig:minnesota} shows inferred traffic flows on the Minnesota road network.

We compare our algorithm against two baselines. 
First, the \emph{ZeroFill} baseline simply assigns $0$ edge flows to all unlabeled edges.
Second, the \emph{LineGraph} baseline uses a line-graph transformation of the network and then applies standard vertex-based SSL on the resulting graph.
More specifically, the original network is transformed into an undirected line-graph, where there is a vertex for each edge in the original network;
two vertices in the line-graph are connected if the corresponding two edges in the original network share a vertex.
Flow values (including sign) on the edges in the original network are the labels on the corresponding vertices in the transformed line-graph.
Unlabeled edge flows are then inferred with a classical vertex-based SSL algorithm on the line-graph~\cite{Zhu_2003}.

\xhdr{Results}
We test the performance of our algorithm \emph{FlowSSL} and the two baseline methods 
for different ratios of labeled edges (\cref{fig:synthetic}).
The LineGraph approach performs no better than ZeroFill. 
This should not be surprising, since the LineGraph approach does not interpret the sign of an edge flow as an orientation but simply as part of a numerical label.
On the other hand, our algorithm out-performs both baselines considerably.
\emph{FlowSSL} works especially well on the Minnesota road network and the Balerma water supply network, 
which have small average degree $\langle d \rangle$.
The intuitive reason is that the dimension of the cycle space is $m-n+1 = n (\langle d \rangle /2 -1) + 1$;
therefore, low-degree graphs have fewer degrees of freedom associated with a zero penalty in the objective 
\cref{eq:objective_divergence_free}.

\subsection{Learning Real-World Traffic Flows \label{subsec:real_world_traffic}}
\begin{figure}[t]
\centering
\includegraphics[width=1.0\linewidth]{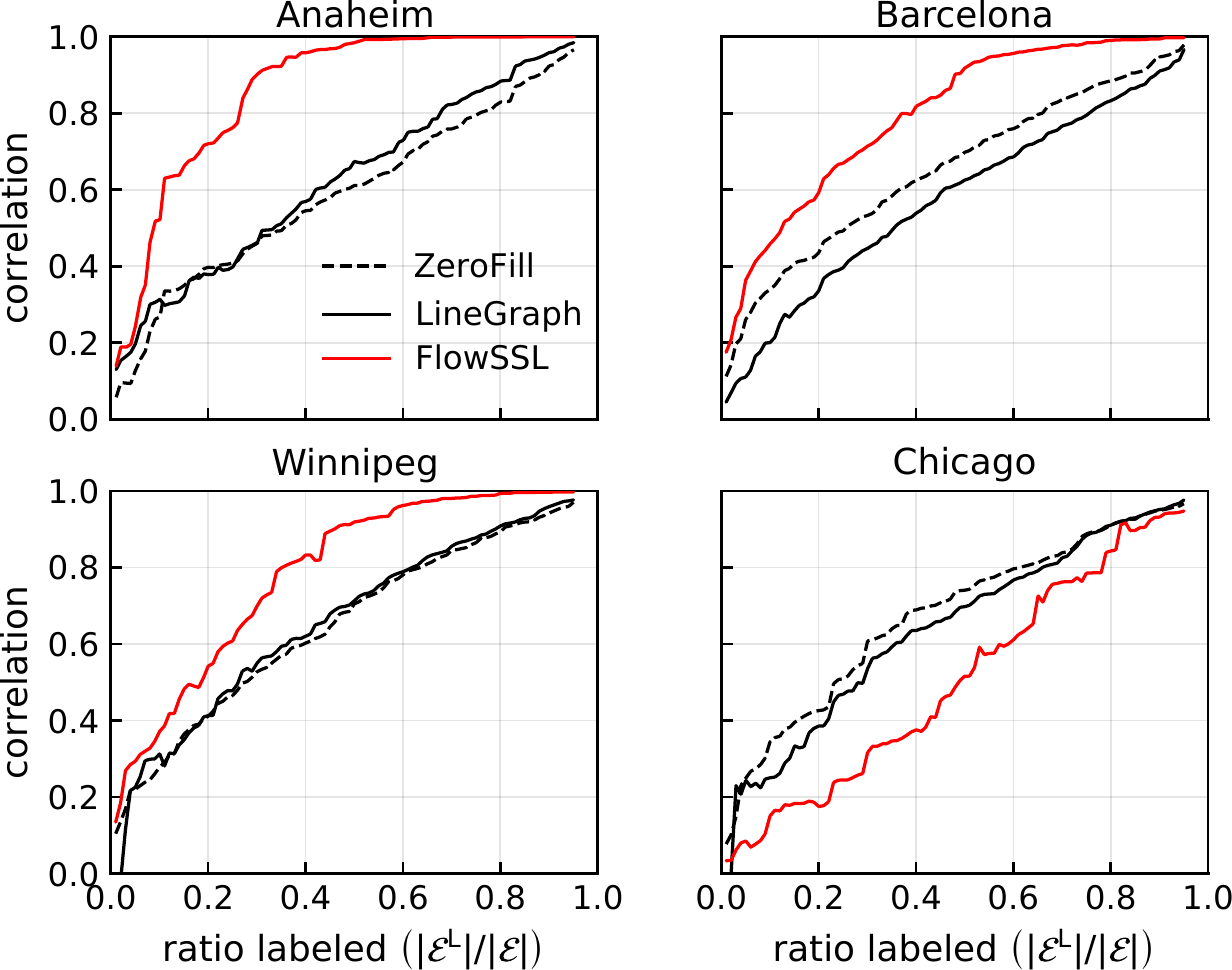}
\caption{Graph-based SSL for real-world traffic flows.
We plot the correlation between the estimated flow $\flowv^{*}$ and the ground truth $\hat{\flowv}$ measured in four transportation networks, as a function of the ratio of labeled edges.
Our FlowSSL outperforms the baselines except in Chicago,
which has a large flow component in the cut space
(\cref{fig:traffic_spectral}).
}

\label{fig:traffic}
\end{figure}
\begin{figure}[t]
\centering
\includegraphics[width=1.0\linewidth]{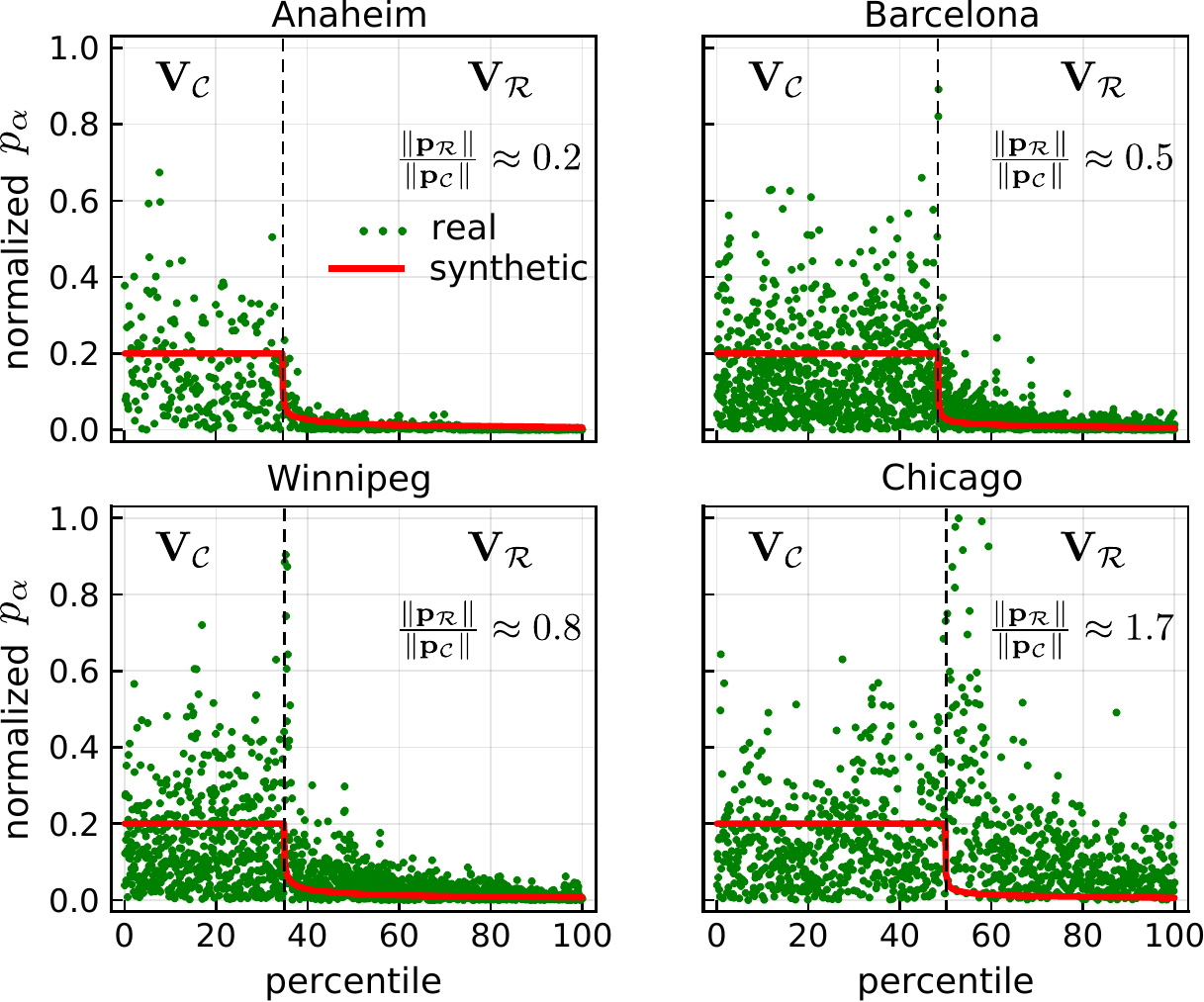}
\caption{%
The normalized spectral coefficients of real-world and synthetic edge flows (\cref{eq:synth_spectral}).
The spectral coefficients are ordered by increasing singular values and plotted as a function of the percentile ranking. 
The basis vectors in the cycle-space $\mathbf{V}_{\mathcal{C}}$ (lower percentile) all have zero singular values.
The real-world traffic flow's spectral coefficients are taken in absolute value and normalized so that the 
root-mean-square of $\mathbf{p}_{\mathcal{C}}$ (the spectral coefficients in cycle-space) equals 0.2. 
The different rates of decay in spectral coefficients leads to different performance of our method (\cref{fig:traffic}).
}
\label{fig:traffic_spectral}
\end{figure}
We now consider actual, measured flows rather than synthetically generated flows.
Accordingly, our assumption of approximately divergence-free flows may or may not be valid.
We consider transportation networks and associated measured traffic flows from four cities (Anaheim, Barcelona, Winnipeg, and Chicago)~\cite{Stabler_2016}.
To test our method, we repeat the same procedure we used for processing synthetic flows in
\cref{subsec:synthetic} with these real-world measured flows.
\Cref{fig:traffic} displays the results.

Our algorithm performs substantially better than the baselines on
three out of four transportation networks with real-world flows.
It performs worse than the baseline on the Chicago road network.
To understand this phenomenon, we compare the spectral coefficients of the real-world traffic flows in four cities with the ``damped-inverse'' synthetic spectral coefficients from \cref{eq:synth_spectral} (see \cref{fig:traffic_spectral}).
We immediately see that the real-world spectral coefficients $\hat{p}_{\alpha}$ do not significantly decay with increasing singular value in the Chicago network, in contrast to the other networks.
Formally, we measure how much the real-world edge flows deviate from our divergence-free assumption by computing the spectral ratio $\|\mathbf{p}_{\mathcal{R}}\| / \|\mathbf{p}_{\mathcal{C}}\|$ between the norms of the spectral coefficients in the cut-space and the cycle-space.
The ratios in the first three cities are all below $1.0$, indicating divergence-free flow is the dominating component.
However, the spectral ratio of traffic flows in Chicago is approximately $1.7$, which explains why our method fails to give accurate predictions.
Moreover, in the Chicago network, the spectral coefficients $\hat{p}_{\alpha}$ with the largest magnitude are actually concentrated in 
the cut-space basis vectors (with smallest singular values).
Later, we show how to improve our results by strategically choosing edges on which to measure flow, 
rather than selecting edges at random (\cref{subsec:active_learning}).

\subsection{Information Flow Networks}
Thus far we have focused on networks embedded in space, where the edges represent some media through which physical units flow between the vertices.
Now we demonstrate the applicability of our method to information networks by considering the edge-based SSL problem for predicting transitions among songs in a user's play sequence on Last.fm\footnote{This dataset is from \url{https://www.last.fm/}.}.
A user's play sequence consists of a chronologically ordered sequence of songs he/she listened to, and a song may repeatedly show up.
Taking the playlist of a user, we represent each unique song as a vertex in graph, and we connect two vertices if they are adjacent somewhere in the playlist.
The ground truth flows is constructed as follows: every time the user plays song $A$ followed by song $B$, add one unit of flow from $A$ to $B$.
We similarly constructed a flow network that records the transition among the artists of songs.
The flow networks constructed here are close to divergence-free, since every time a user transitions to a song or artist, he/she typically transition out by listening to other ones.
We used the same set of experiments to evaluate flow prediction for these networks (\cref{fig:flow_misc}).
Our method outperforms the baselines, despite the flows are not from physical systems.

\begin{figure}[t]
\includegraphics[width=1.0\linewidth]{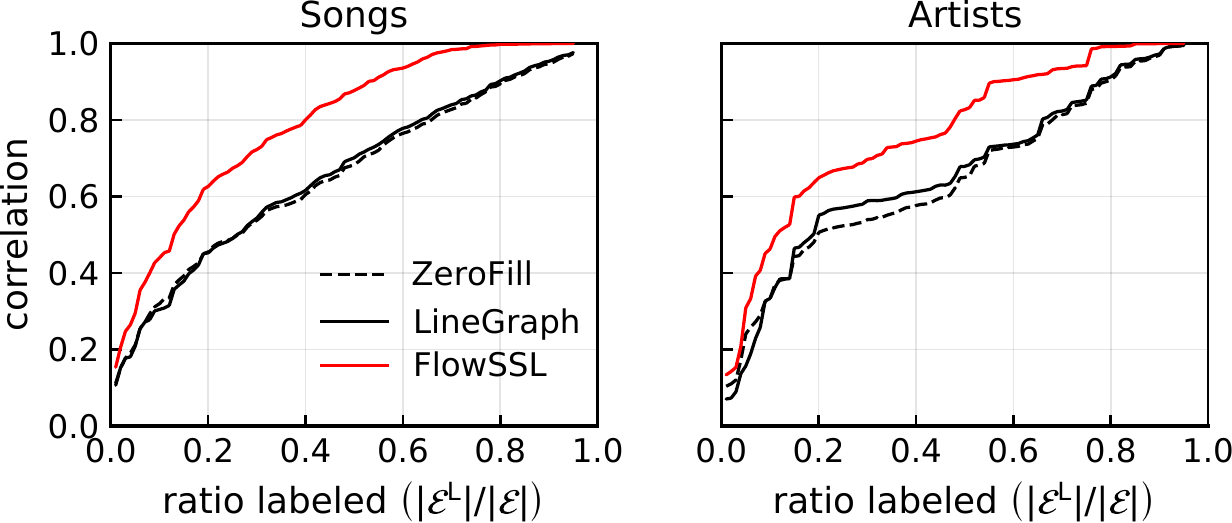}
\caption{Graph-based SSL for real-world flows among songs and artists in a music playlists.
The plots show the correlation between the estimated flow
vector $\flowv^{*}$ and the ground truth $\hat{\flowv}$ on flow networks
of songs and artists on a music streaming service. Even though
the flows are not physical, FlowSSL method still outperforms the baselines.
}
\label{fig:flow_misc}
\end{figure}

\section{Active Semi-Supervised Learning \label{sec:acl}}
We now focus on the problem of selecting the set of labeled edges
that is most helpful in determining the overall edge flows in a network.
While selecting the most informative set of labeled vertices has been well
studied in the context of vertex-based semi-supervised
learning~\cite{Gadde_2014,Guillory_2009}, active learning in the edge-space 
remains largely under-explored.
%
%
Traffic flows are typically monitored by road detectors, but installation and
maintenance costs often prevent the deployment of these detectors on the entire
transportation network.
In this scenario, solving the active learning problem in the edge-space enables us to choose an optimal set of roads to deploy sensors under a limited budget.

\subsection{Two active learning algorithms \label{subsec:active_learning}}

We develop two active semi-supervised learning algorithms for selecting edges to measure.
These algorithms improve the robustness of our method for learning edge flows.

\xhdr{Rank-revealing QR (RRQR)}
According to \cref{lma:reconstruction_error}, the upper bound of the
reconstruction error decreases as the smallest singular value of
$\mathbf{V}_{\mathcal{C}}^{\rm L}$ increases.
Therefore, one strategy for selecting $\mathcal{E}^{\rm L}$ is to choose
$m^{\rm L}$ rows from $\mathcal{V}_{0}$ that maximize the smallest singular value of
the resulting submatrix.
This problem is known as optimal column subset selection (maximum submatrix volume) and is NP-hard~\cite{Civril_2009}. 
However, a good heuristic is the rank revealing QR decomposition (RRQR)~\cite{Chan_1987}, which computes
\begin{align}
\mathbf{V}_{\mathcal{C}}^{\intercal}\ \Pi = Q 
\begin{bmatrix}
R_{1} & R_{2}
\end{bmatrix}.
\end{align}
Here, $\Pi$ is a permutation matrix that keeps $R_{1}$ well-conditioned.
Each column permutation in $\Pi$ corresponds to an edge, and the resulting edge set $\mathcal{E}^{\rm L}$ for active learning chooses the first $m^{\rm L}$ columns of $\Pi$.
This approach is mathematically similar to graph clustering algorithms that use RRQR to select representative vertices for cluster centers~\cite{Damle_2016}.
The computational cost of RRQR is $\mathcal{O}(m^{3})$.

\xhdr{Recursive Bisection (RB)}
In many real-world flow networks, there exist a global trend of flows across different cluster of vertices.
For example, traffic during morning rush hour flows from rural to urban regions
or electricity flows from industrial power plants to residential households.
%
The spectral projection of such global trends is concentrated on singular vectors $\mathbf{v}\in$
$\mathbf{V}_{\mathcal{R}}$ with small singular values corresponding to gradient flows
(e.g., $\mathbf{v}_{6}$ in \cref{fig:spectral}), as was the case with the Chicago traffic flows (\cref{fig:traffic_spectral}).

Building on this observation, our second active learning algorithm uses a heuristic recursive bisection (RB) approach for selecting labeled edges.\footnote{We call this algorithm recursive \emph{bisection} although it does not necessarily gives two clusters with the same number of vertices.}
%
%
The intuition behind this heuristic is that edge flows on bottleneck-edges, which partition a network, are able to capture global trends in the networks' flow pattern.
We start with an empty labeled set $\mathcal{E}^{\rm L}$, a target number of
labeled edges $m^{\rm L}$, and the whole graph as one single cluster.
Next, we recursively partition the largest cluster in the graph with spectral
clustering and add every edge that connects the two resulting clusters into
$\mathcal{E}^{\rm L}$, until we reach the target number of labeled edges.
Similar methods have been shown to be effective in semi-supervised active learning for
vertex labels~\cite{Guillory_2009}; in these cases, the graph is first
clustered, and then one vertex is selected from each cluster.
While any other graph partitioning algorithm could be used and greedy recursive bisection approaches 
can be sub-optimal~\cite{Simon_1998}, we find that this simple methods works well in practice on our datasets, and its iterative 
nature is convenient for selecting a target number of edges.
The computational cost of the recursive bisection algorithm is $\mathcal{O}(m \log n)$.

\subsection{Results}
\begin{figure}[t]
\centering
\includegraphics[width=1.0\linewidth]{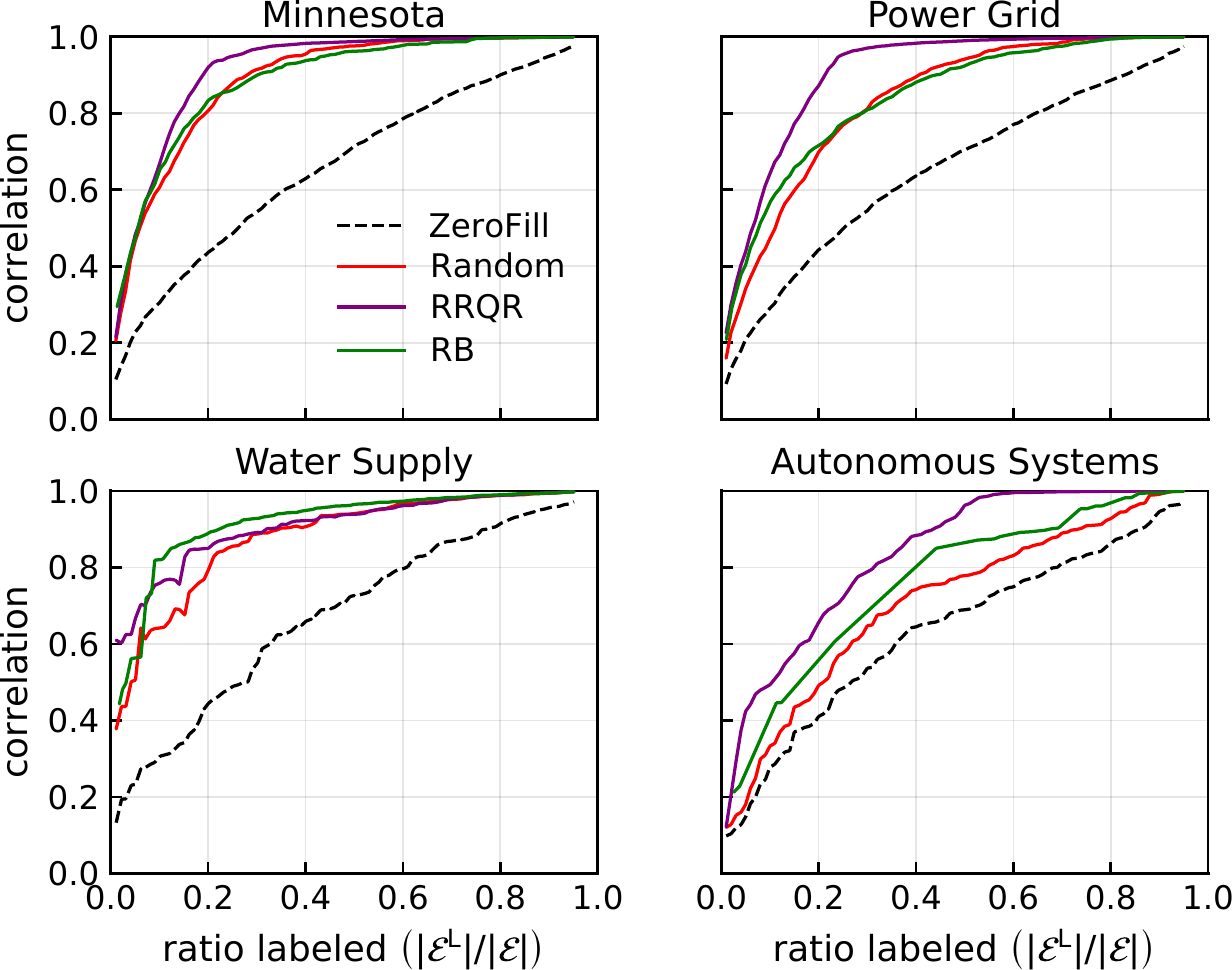}
\caption{Graph-based semi-supervised active learning for \emph{synthetic} flows. 
The plots show the Pearson correlation coefficients between the estimated flow
vector $\flowv^{*}$ and the synthetic ground truth edge flows $\hat{\flowv}$ as
a function of the ratio of labeled edges. Our rank-revealing QR (RRQR)
active learning performs well on synthetic datasets.
}
\label{fig:active_synthetic}
\end{figure}
\begin{figure}[t]
\centering
\includegraphics[width=1.0\linewidth]{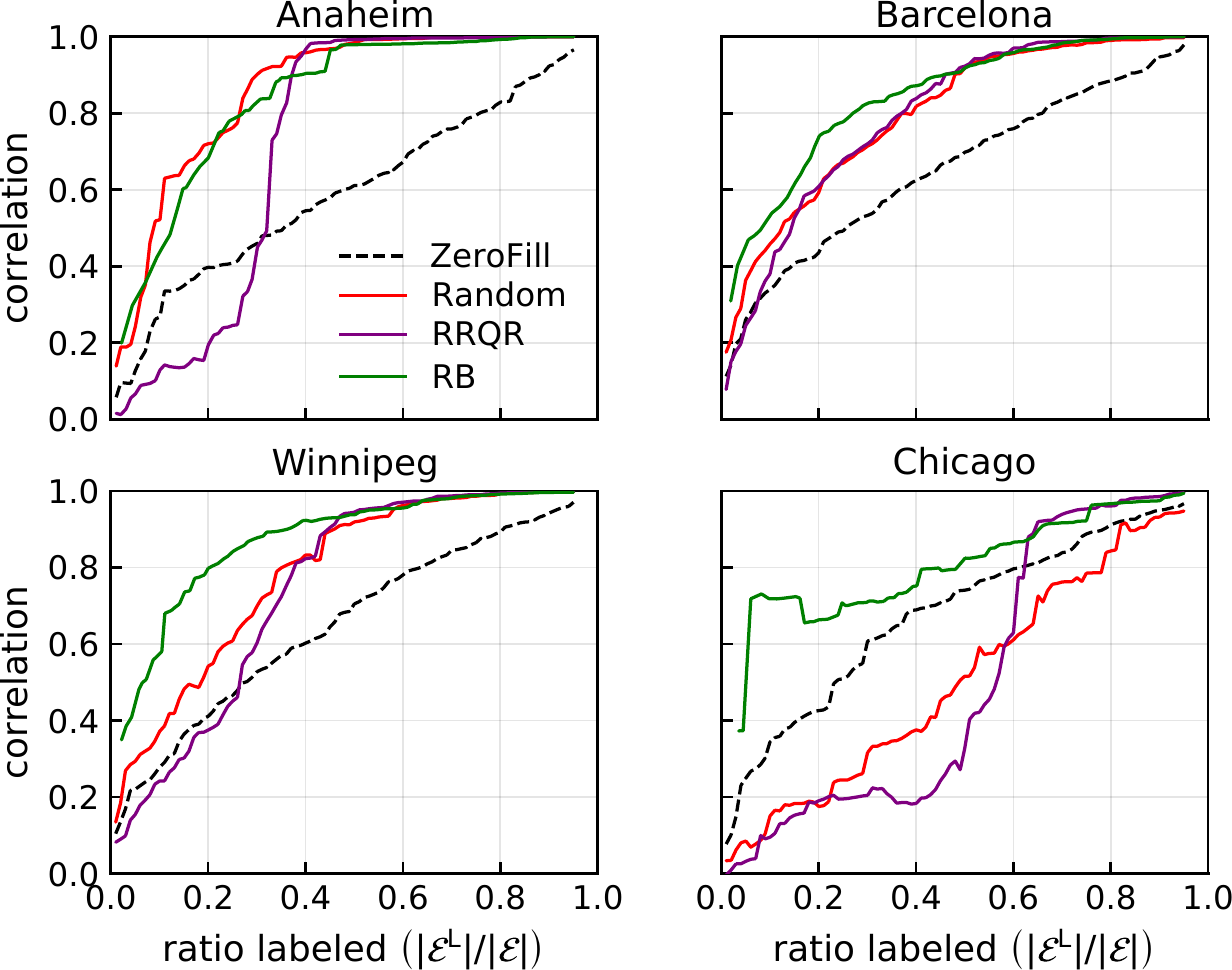}
\caption{Graph-based semi-supervised active learning for \emph{real-world} traffic flows.
The plots show the Pearson correlation coefficients between the estimated flow
vector $\flowv^{*}$ and the ground truth $\hat{\flowv}$ measured in four
transportation networks, as a function of the ratio of labeled edges.
With our Recursive Bisection (RB) active learning method to select edges, we now
perform better than the baseline (ZeroFill) on the Chicago traffic dataset
(cf.\ \cref{fig:traffic}).
}
\label{fig:active_real_world}
\end{figure}

\begin{figure}[t]
\centering
    \includegraphics[height=3.8cm]{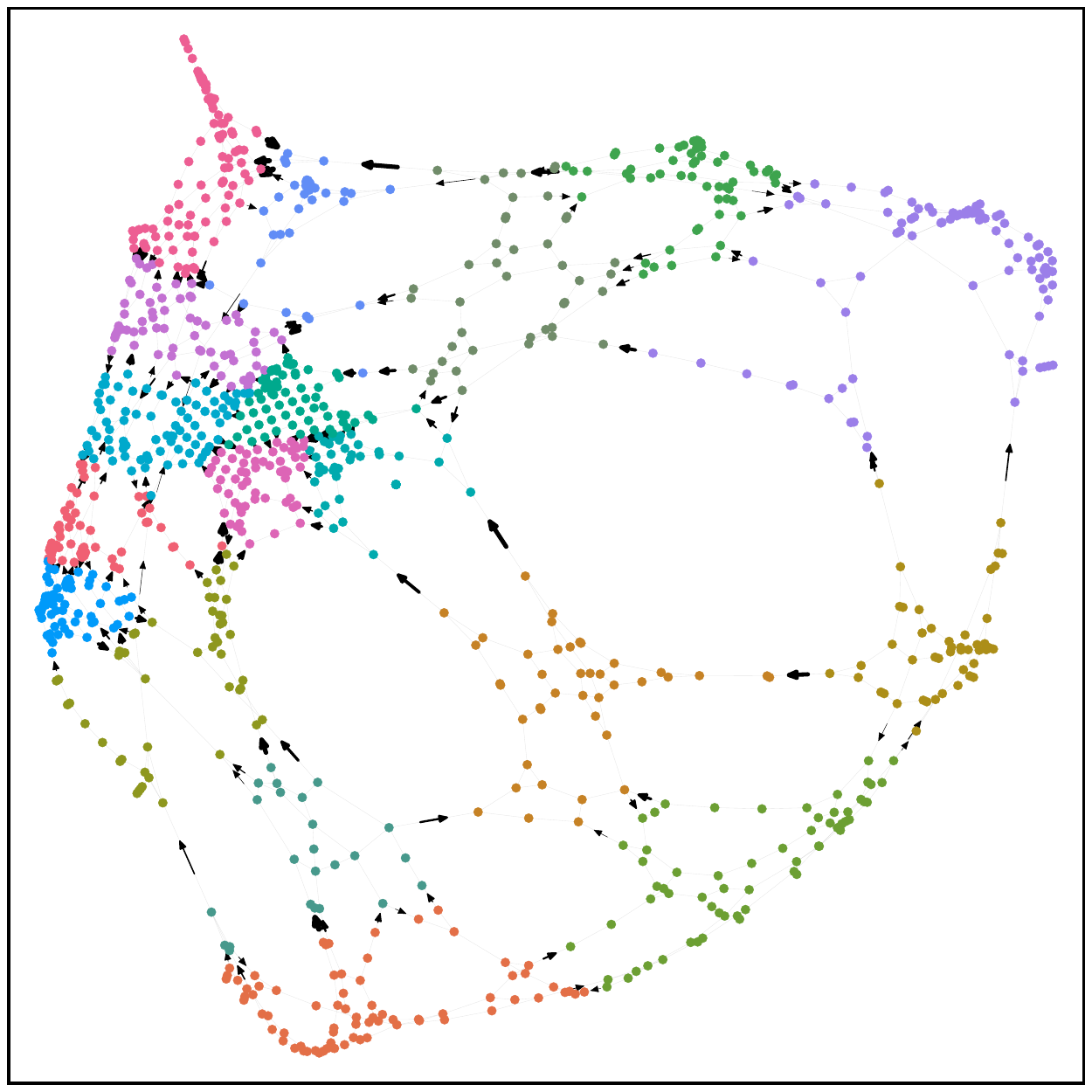}
    \hspace{1.7mm}
    \includegraphics[height=3.8cm]{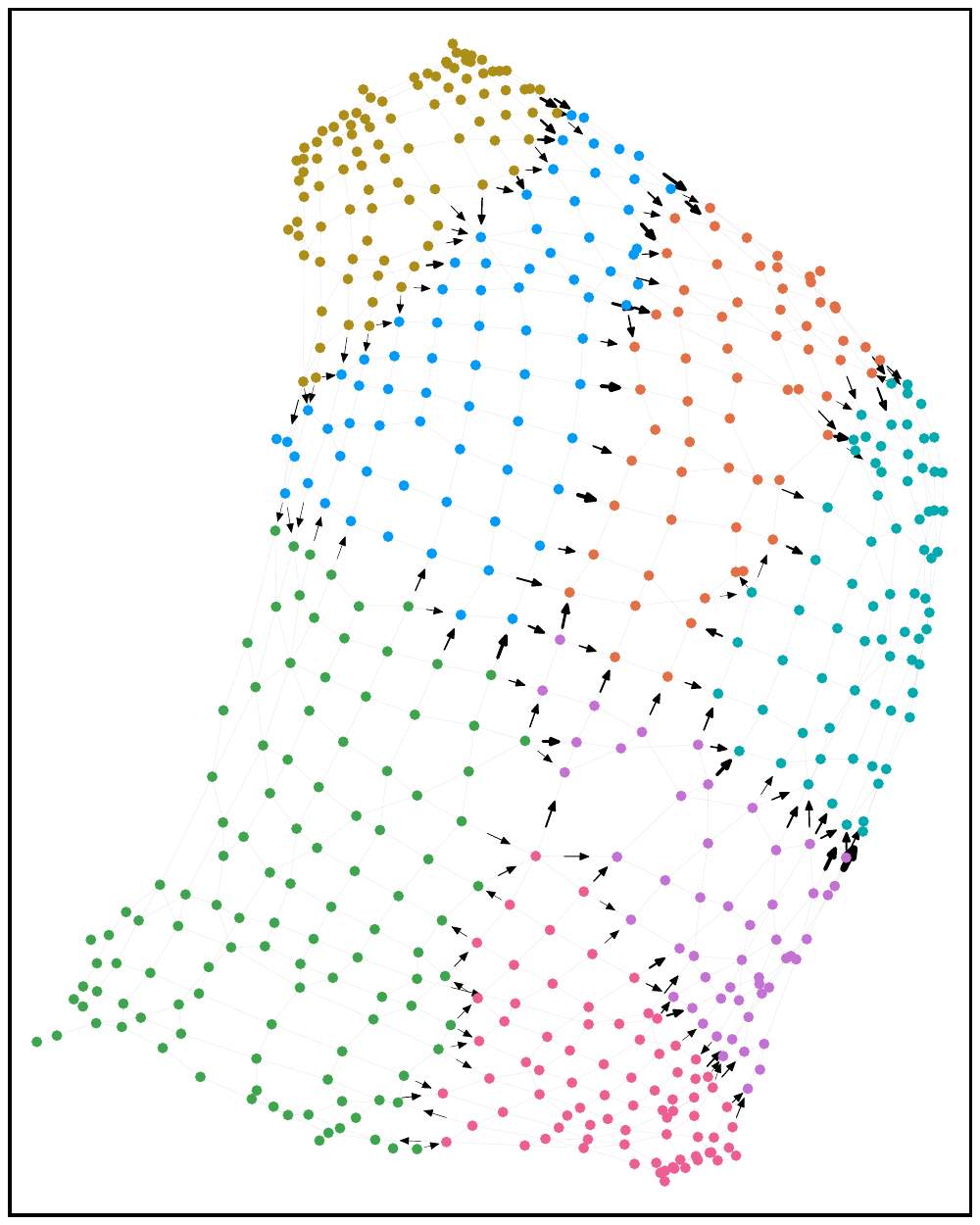}
\caption{The clusters discovered by our recursive bisection algorithm in Winnipeg (left) and Chicago (right) road networks, where the vertex coordinates are computed by spectral-embedding.
In each network, $10\%$ of the edges are selected as the labeled set $\mathcal{E}^{\rm L}$ and the ground truth edge flows on those edges are plotted as black arrows.
}
\label{fig:clustering}
\end{figure}

We repeat the experiments in \cref{sec:ssl_result} on traffic networks with labeled edges
$\mathcal{E}^{\rm L}$ selected by our RRQR and RB algorithms.
For comparison, the ZeroFill approach with randomly selected labeled edges is included as a baseline.
Our RRQR algorithm outperforms both recursive bisection and random selection for
networks with synthetic edge flows, where the divergence-free condition on vertices
approximately holds (\cref{fig:active_synthetic}).
However, for networks with real-world edges flows, RRQR performs
poorly, indicating that the divergence-free condition is too strong an assumption (\cref{fig:active_real_world}).
In this case, our RB method consistently outperforms the baselines, especially for small numbers of labels.

To provide additional intuition for the RB algorithm, we plot the selected labeled edges
and the final clusters in the Winnipeg and Chicago network when
allowing $10\%$ of edges to be labeled (\cref{fig:clustering}).
The correlation coefficients resulting from random edge selection and RB
active learning are $\rho_{\rm rand} = 0.371$ and $\rho_{\rm RB} = 0.580$ for the Winnipeg road network, respectively.
For the Chicago road network we obtain correlations of $\rho_{\rm rand} = 0.151$ and $\rho_{\rm RB} =0.718$.
Thus, the active learning strategy alleviates our prior issues with learning on the Chicago road network 
by strategically choosing where to measure edge flows.
%


\section{Extensions to cyclic-free flows}\label{sec:topology}
Thus far, our working assumption has been that edge flows are approximately divergence-free.
However, we might also be interested in the opposite scenario, 
if we study a system in which circular flows should \emph{not} be present.
Below we view our method through the lens of basic combinatorial Hodge theory. 
This viewpoint illuminates further how our previous method is a projection onto the space of divergence-free edge flows.
By projecting into the complementary subspace, we can therefore learn flows that are (approximately) cycle-free.
We highlight the utility of this idea through a problem of fair pricing of foreign currency exchange rates, 
where the cycle-free condition eliminates arbitrage opportunities.

\subsection{The Hodge decomposition}

Let $\flowv$ be a vector of edge flows.
The \emph{Hodge decomposition} provides an orthogonal decomposition of $\flowv$~\cite{Lim_2015,schaub2018random}:
\begin{align}
\underbrace{\flowv}_{\text{edge flow}} = 
\overbrace{\mathbf{B}^{\intercal}\mathbf{y}}^{\text{gradient flow}}
\oplus
\phantom{!!!}
\overbrace{
  \underbrace{\mathbf{C}\mathbf{w}}_{\text{curl flow}}
  \phantom{!!!}
  \oplus
  \underbrace{\mathbf{h}}_{\text{harmonic flow}}
}^{\text{divergence-free flow}}
\label{eq:hodge}
\end{align}
where the matrix $\mathbf{C} \in \mathbb{R}^{m \times o}$ (called the curl operator) maps edge flows to curl around a triangle,
\begin{align}
C_{ru} =
\begin{cases}
 1, & \mathcal{T}_{u} \equiv (i,j,k),\ \mathcal{E}_{r} \equiv (i,j),\ i < j < k \\
 1, & \mathcal{T}_{u} \equiv (i,j,k),\ \mathcal{E}_{r} \equiv (j,k),\ i < j < k \\
-1, & \mathcal{T}_{u} \equiv (i,j,k),\ \mathcal{E}_{r} \equiv (i,k),\ i < j < k \\
 0, & \text{otherwise}
\end{cases}
\label{eq:curl_matrix}
\end{align}
and $(\mathbf{B}^{\intercal}\mathbf{B} + \mathbf{C}\mathbf{C}^{\intercal})\mathbf{h} = 0$.
Here, the gradient flow component is zero if and only if
$\flowv$ is a divergence-free flow.
Thus far, we have focused on controlling the gradient flow;
specifically, the objective function in \cref{eq:objective_divergence_free}
penalizes a solution $\flowv$ where $\| \mathbf{B}\flowv \|$ is large.

We can alternatively look at other components of the flow given by the Hodge decomposition.
In \cref{eq:hodge}, the ``curl flow'' captures all flows that can be composed of flows around
triangles in the graph.
This component is zero if the sum of edge flows given by $\flowv$ around every triangle is 0. 
Combining \cref{eq:curl_matrix,eq:flow_function}, the curl of a flow on triangle $\mathcal{T}_{u} = (i, j, k)$ with oriented edges $(i, j)$, $(j, k)$, and $(i, k)$ is
$[\mathbf{C}^{\intercal}\flowv]_u = f(i, j) + f(j, k) - f(i, k) = f(i, j) + f(j, k) + f(k, i)$.

Finally, the vector $\mathbf{h}$ is called the harmonic flow and measures flows
that cannot be constructed from a linear combination of gradient and curl flows.
Projecting edge flows onto the space of gradient flows is the HodgeRank method for ranking with pairwise comparisons~\cite{Jiang-2010-HodgeRank}.
In the next section, we use $\| \mathbf{C}^{\intercal}\flowv \|$ as part of an objective function to alternatively learn flows that have small curl.

\subsection{An application to Arbitrage-Free Pricing \label{subsec:consistent_pricing}}
\begin{figure}[t]
\centering
\begin{minipage}[c]{0.6\columnwidth}
\includegraphics[width=1.0\columnwidth]{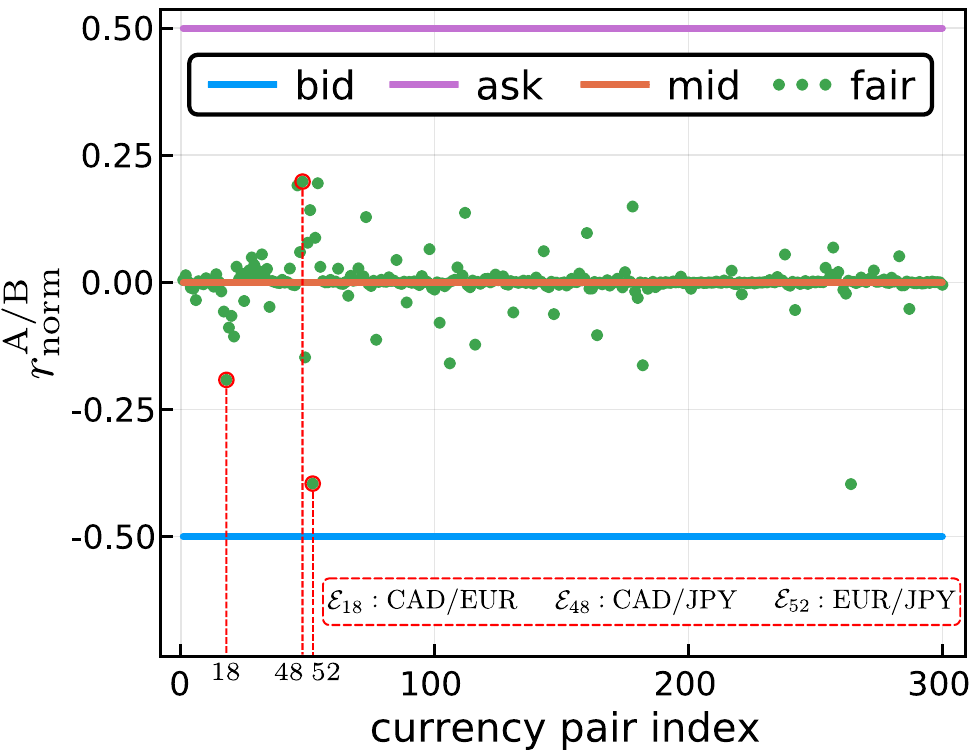}
\end{minipage}
\hfill
\begin{minipage}[c]{0.36\columnwidth}
\caption{Exchange rates and fair rates in a currency exchange market.
%
Rates are normalized so that bid and ask are 0.5 and -0.5.
We find fair trading prices by penalizing curl flow in the exchange.
\label{fig:fx}
}
\end{minipage}
\end{figure}

We demonstrate an application of edge-based learning with a
different type of flow constraint. 
In this case study, edge flows are currency exchange rates, where participants buy, sell, exchange, and speculate on foreign currencies.
Every pair of currencies has two exchange rates: the bid is the rate at which the market is prepared to buy a specific currency pair.
The ask is the rate at which the market is prepared to sell a specific currency pair.
There is also a third widely used exchange rate called the ``middle rate,'' which is
the average of the bid and ask, is often used as the price to facilitate a trade between currencies.
%
An important principle in an efficient market is the no-arbitrage condition,
which states that it is not possible to obtain net gains by a sequence of
currency conversions.
 Although the middle rates are widely accepted as a ``fair'' rate,
they do not always form an arbitrage-free market.
For example, in a dataset of exchange rates between the 25 most traded currencies at
2018/10/05 17:00 UTC~\cite{Oanda_2018}, the middle rates for CAD/EUR, EUR/JPY and JPY/CAD were
$0.671200$, $130.852$ and $0.0113876$, respectively. Therefore, a trader
successively executing these three transactions would yield
$1.000146$ CAD from a 1 CAD investment.

Here we show how to construct  a ``fair'' set of exchange rates that is arbitrage-free.
We first encode the exchange rates as a flow network.
Each currency is a vertex, and for each pair of currencies $A$ and
$B$ with exchange rate $r^{\rm A/B}$, we connect $A$ and $B$ with $\log(r^{\rm A/B})$
units of edge flow from $A$ to $B$, which ensures that $f(A, B) = -f(B, A)$.
The resulting exchange network is fully connected.
Under this setup, the arbitrage-free condition translates into requiring the
edge flows in the exchange network to be cycle-free.

We can constrain the edge flows on every triangle to sum to $0$ by the
curl-free condition $\|\mathbf{C}^{\intercal} \flowv\| = 0$.
Moreover, in a fully connected network, curl-free flows are cycle-free.
%
Thus, we propose to set fair rates by minimizing the curl over all triangles,
subject to the constraint that the fair price lies between the bid and ask prices:
\begin{align}
\flowv^{*} &= \arg \min_{\flowv} \|\mathbf{C}^{\intercal} \flowv\|^{2} + \lambda^{2} \cdot \|\flowv - \flowv^{\rm mid}\|^{2} \hspace{0.05in} \text{s.t.} \hspace{0.05in} \flowv^{\rm bid} \leq \flowv \leq \flowv^{\rm ask}.
\label{eq:objective_curl_free}
\end{align}
The second term in the objective ensures that the minimization problem is
not under-determined.
In our experiments, $\lambda = 1.0 \cdot 10^{-3}$ and we solve the convex
quadratic program with linear constraints using
the Gurobi solver. Unlike the middle rate, which
is computed only using the bid/ask rates of one particular currency pair, the
solution to the optimization problem above accounts for \emph{all} bid/ask rates
in the market to determine the fair exchange rate.
\Cref{fig:fx} shows the fair exchange rates, and the computed rates for
CAD/EUR, EUR/JPY and JPY/CAD are $0.671169$, $130.845$ and $0.0113871$, respectively,
removing the arbitrage opportunity.

\section{Related Work \label{sec:rw}}
The Laplacian $\mathbf{L}$ appears in many graph-based SSL algorithms to enforce smooth signals in the vertex space of the graph.
Gaussian Random Fields~\cite{Zhu_2003} and Laplacian Regulation~\cite{Belkin_2006} are two early examples, 
and there are several extensions~\cite{Wu_2012,Solomon_2014}. 
However, these all focus on learning vertex labels and, as we have seen, directly applying ideas from vertex-based SSL to learn edge flows on the line-graph does not perform well.
In the context of signal processing on graphs, there exist preliminary edge-space analysis~\cite{Schaub2018,barbarossa2016introduction,Zelazo_2011}, 
but semi-supervised or active learning are not considered.

In terms of active learning, several graph-based active semi-supervised algorithms have been designed for learning vertex labels, based on error bound minimization~\cite{Gu_2012}, 
submodular optimization~\cite{Guillory_2011}, or
variance minimization~\cite{Ji_2012}.
Graph sampling theory under spectral assumptions has also been an effective strategy~\cite{Gadde_2014}.
Similarly, in \cref{subsec:theory}, we give exact recovery conditions and derive
error bounds for our method assuming the spectral coefficients of the basis vectors
representing potential flows are approximately zero, which motivated our use of
RRQR for selecting representative edges.
There are also clustering heuristics for picking vertices~\cite{Guillory_2009};
in contrast, we use clustering to choose informative edges.


\section{Discussion \label{sec:discussion}}
We developed a graph-based semi-supervised learning method for edge flows.
Our method is based on imposing interpretable flow constraints to
reflect properties of the underlying systems.
These constraints may correspond to enforcing divergence-free flows in the case of flow-conserving transportation systems,
or non-cyclic flows as in the case of efficient markets.
Our method permits spectral analysis for deriving exact recovery condition and bounding reconstruction error, provided that the edge flows are indeed (nearly) divergence free.
%
%
%
On a number of synthetic and real-world problems, our method substantially outperforms competing baselines.
Furthermore, we explored two active semi-supervised learning algorithms for edge flows. 
The RRQR strategy works well for synthetic flows, while a recursive partitioning approach works well on real-world datasets.  
The latter result hints at additional structure in the real-world data that we can exploit for better algorithms.
%


\section*{Acknowledgements}
This research was supported in part by European Union's Horizon 2020 research and innovation programme under the Marie Sklodowska-Curie grant agreement No 702410; NSF Award DMS-1830274; and ARO Award W911NF-19-1-0057.

\bibliographystyle{ACM-Reference-Format}
\bibliography{main}
\setcitestyle{unsort}
\appendix
\clearpage

\section{Appendix \label{sec:supplementary}}
Here we provide some implementation details of our method to help readers 
reproduce and further understand the algorithms and experiments in this paper.
First, we present the solvers for learning edge flows in divergence-free and curl-free networks.
Then, we further discuss the active learning algorithms used for selecting informative edges.
All of the algorithms used in this paper are implemented in Julia 1.0.

\subsection{Graph-Based SSL for Edge Flows}
We created a Julia module \emph{NetworkOP} for processing edge flows in networks.
It contains a \emph{FlowNetwork} class which records the vertices, edges and triangles of a graph as three ordered dictionaries.\footnote{Although we considered unweighted
  graphs in this work, we choose ordered dictionaries over lists for future exploration of weighted graphs.}
The NetworkOP also provides many convenience functions, such as computing the incident matrix of a FlowNetwork object.
Such functions greatly simplify our implementations for learning unlabeled edge flows.
In this part, we assume the labeled edges are given. In \cref{subsec:supplementary_active_learning} we show algorithms for selecting edges.
\begin{figure}[H]
\scalebox{0.78}{
\begin{lstlisting}
using SparseArrays, NetworkOP;
using LinearMaps, IterativeSolvers;

function ssl_df(A::SparseMatrixCSC{Float64,Int64},
                F::SparseMatrixCSC{Float64,Int64},
                IdU::Vector{Int64}, lambda=1.0e-1)
  # A: adjacent matrix of the graph
  # F: anti-symmetric matrix for ground truth flows
  # IdU: list of indices for unlabeled edges
  
  # create a flow network object
  FN = NetworkOP.FlowNetwork(A);
  n = length(FN.VV); # n: number of vertices
  m = length(FN.EE); # m: number of edges
  
  # assemble edge flows to vector
  fhat = NetworkOP.mat2vec(FN,F);
  # the trival solution for Eq.(6)
  f0 = collect(fhat); f0[IdU] = zeros(length(IdU));
  # the incidence matrix for network
  B = NetworkOP.mat_div(FN);
  # expansion operator \Psi and its transpose
  expand = ff->collect(sparsevec(IdU,ff,m));
  select = ff->ff[IdU];
  
  # the operator in iterative least-squares problem
  map = LinearMap{Float64};
  op = map(ff->vcat(B*expand(ff), lambda*ff),
           pp->select(B'*pp[1:n]) + lambda*pp[n+1:end],
           n+length(IdU), length(IdU);
           ismutating=false);
  # infer edge flows on unlabeled edges with LSQR
  fU = lsqr(op, vcat(-B*f0, zeros(length(IdU))));
  # fstar: the inferred edge flow vector
  fstar = f0 + expand(fU);

  return fstar;
end
\end{lstlisting}}
\caption{Code snippet for inferring edge flows in a divergence-free network by solving a least-squares problem.}
\label{fig:supplementary_divergence_free}
\end{figure}

\xhdr{Learning flows in divergence-free networks}
Our algorithm for solving \cref{eq:objective_divergence_free} is presented in \cref{fig:supplementary_divergence_free}.
This algorithm takes three arguments as input:
(1) $\mathbf{A} \in \mathbb{R}^{n \times n}$ is the adjacency matrix of a graph;
(2) $\mathbf{F} \in \mathbb{R}^{n \times n}$ is the matrix containing ground truth edge flows, where $\mathrm{F}_{ij} = f(i,j) = -\mathrm{F}_{ji}$; and
(3) $\mathbf{IdU} \in \mathbb{N}^{m^{\rm U}}$ is the edge indices of unlabeled edges.
The edge flow matrix $\mathbf{F}$ is transformed into an edge flow vector $\hat{\flowv}$ with the \texttt{mat2vec} function (line 17).
Following the formulation in \cref{eq:least_square}, we define the incidence matrix $\mathbf{B}$ (line 21), the expansion operator $\mathbf{\Phi}$ and its transpose (line 23,24), then we assemble them into a linear map (line 27-31) as the main operator in the least-squares problem.
Finally, we solve the least-squares problem with an iterative LSQR solver from the \texttt{IterativeSolvers.jl} package.

\xhdr{Learning flows in cycle-free networks}
\begin{figure}[t]
\raggedright
\scalebox{0.78}{
\begin{lstlisting}
using SparseArrays, NetworkOP;
using JuMP, Gurobi;

function ssl_cf(A::SparseMatrixCSC{Float64,Int64},
                bid::Vector{Float64},
                mid::Vector{Float64},
                ask::Vector{Float64}, lambda=1.0e-3)
  # A: adjacent matrix of the graph
  # bid: flow vector f^{bid} representing bid rate
  # mid: flow vector f^{mid} representing middle rate
  # ask: flow vector f^{ask} representing ask rate
  
  # create a flow network object
  FN = NetworkOP.FlowNetwork(A);
  n = length(FN.VV); # n: number of vertices
  m = length(FN.EE); # m: number of edges
  # map from edge to edge-index
  e2id = Dict(e=>i for (i,e) in enumerate(keys(FN.EE)));
  
  model = Model(solver=GurobiSolver(Presolve=0));
  # variables are the fair exchange rates
  @variable(model, bid[i] <= f[i=1:m] <= ask[i]);
  # objective function in Eq.(20)
  @objective(model, Min, 
    sum((f[e2id[(i,j)]]+f[e2id[(j,k)]]-f[e2id[(i,k)]])^2
        for i=1:n, j=i+1:n, k=j+1:n) + 
    sum((f[i]-mid[i])^2 for i=1:m))*lambda^2;
  
  status = solve(model);
  # fair: flow vector f^{fair} representing fair rate
  fair = getvalue(f);

  return fair;
end
\end{lstlisting}}
\caption{Code snippet for fair pricing in a arbitrage-free foreign exchange network by solving a quadratic program with linear constraints (QPLC).}
\label{fig:supplementary_cycle_free}
\end{figure}
Our algorithm for computing the fair rate in a foreign exchange market is given in \cref{fig:supplementary_cycle_free}.
It takes as input the (fully connected) adjacency matrix $\mathbf{A}$ as well as the flow vectors $\mathbf{bid}, \mathbf{mid}, \mathbf{ask} \in \mathbb{R}^{m}$ representing the corresponding logarithmic exchange rates between currencies.
Then we use the \texttt{JuMP.jl} package to set up the optimization problem in \cref{eq:objective_curl_free}.
The linear constraints and quadratic objective function are specified in line 22 and 24-27 respectively.
Finally we solve the QPLC problem with the \texttt{Gurobi.jl} package (line 29).

\subsection{Active Learning Strategies \label{subsec:supplementary_active_learning}}
Now we look at the implementation of the two active learning algorithms.
Given the adjacency matrix $\mathbf{A}$ as input, those active learning algorithms output the selected indices $\mathbf{IdL} \in \mathbb{N}^{m^{\rm L}}$ of edges to be labeled.

\xhdr{RRQR algorithm}
\begin{figure}[t]
\scalebox{0.78}{
\begin{lstlisting}
using LinearAlgebra, NetworkOP;

function al_rrqr(A::SparseMatrixCSC{Float64,Int64},
                 ratio::Float64)
  # A: adjacency matrix of the graph
  # ratio: ratio of labeled edges
  
  # create a flow network object
  FN = NetworkOP.FlowNetwork(A);
  n = length(FN.VV); # n: number of vertices
  m = length(FN.EE); # m: number of edges
  # the incidence matrix for network
  B = NetworkOP.mat_div(FN);
  
  # basis for cyclic edge flows
  VC = nullspace(B);

  # permutation order from RRQR
  od = qr(VC',Val(true)).p;
  # the indices for labeled edges
  IdL = od[1:Int64(ceil(m*ratio))];

  return IdL;
end
\end{lstlisting}}
\caption{Code snippet of the RRQR active learning algorithm for selecting informative edges.}
\label{fig:supplementary_rrqr}
\end{figure}
We present our RRQR active learning algorithm in \cref{fig:supplementary_rrqr}.
First, we compute an orthonormal basis $\mathbf{V}_{\mathcal{C}}$ for the cycle-space $\mathcal{C} = \ker(\mathbf{B})$ representing cyclic edge flows (line 16).
Then, we perform pivoted QR decomposition on the rows of $\mathbf{V}_{\mathcal{C}}$ (line 19), and the edge indices for labeled edges are given by the first $m^{\rm L}$ permutations (line 21).

\xhdr{Recursive bisection algorithm}
\begin{figure}[t]
\scalebox{0.78}{
\begin{lstlisting}
using Clustering, NetworkOP;

function al_rb(A::SparseMatrixCSC{Float64,Int64},
               ndim::Int64,
               ratio::Float64)
  # A: adjacency matrix of the graph
  # ndim: number of vectors in spectral clustering
  # ratio: ratio of labeled edges
  
  # create a flow network object
  FN = NetworkOP.FlowNetwork(A);
  m = length(FN.EE); # m: number of edges
  
  # compute the embedded vertex coordinates
  X = spectral_embedding(FN, dims);
  
  IdL = []; # the indices for labeled edges
  clusters = [[1:m]]; # list with current clusters
  # map from edge to edge index
  e2id = Dict(e=>i for (i,e) in enumerate(keys(FN.EE)));
  
  while (length(IdL) < ratio*m)
    # find the cluster with maximum cardinality
    max_id = argmax([length(cc) for cc in clusters]);
    cluster = clusters[max_id];
    Xc = X[:,cluster];
    km = kmeans(Xc,2,init=:kmpp); # perform k-means
    cid = assignments(km); # get cluster assignment
    # split the cluster into two, update cluster list
    clusters[max_id] = cluster[cid .== 1];
    push!(clusters, cluster[cid .== 2]);
    # add the edges connecting two clusters to IdL
    for i in clusters[max_id]
      for j in clusters[end]
        e = i < j ? (i,j) : (j,i);
        if (e in keys(e2id))
          push!(IdL, e2id[e]);
        end
      end
    end
  end
end
\end{lstlisting}}
\caption{Code snippet of the recursive bisection active learning algorithm for selecting informative edges.}
\label{fig:supplementary_rb}
\end{figure}
We present our recursive bisection active learning algorithm in \cref{fig:supplementary_rb}.
This algorithm first uses a spectral embedding to compute the vertex coordinates (line 15).
Then it repeatedly chooses the largest cluster in the graph (line 24-26), uses
the k-means (Lloyd's) algorithm to divide the chosen cluster into two (line 27-31),
and adds the edges that connect the two resulting clusters into the labeled
edges indices (line 33-40).
Note that the k-mean algorithm sometimes fails due to bad initial cluster
centers or vertices with same embedded coordinates; however, we omit the code
for dealing with those corner cases here due to limited space.

The complete implementation of all of our algorithms can be found at \url{https://github.com/000Justin000/ssl_edge.git}.

\end{document}